\documentclass[conference]{IEEEtran}
\IEEEoverridecommandlockouts
\usepackage{cite}
\usepackage{amsmath,amssymb,amsfonts}
\usepackage{textcomp}
\usepackage{algorithm}
\usepackage{algorithmic}
\usepackage{graphicx}
\usepackage{subfigure}
\usepackage{amssymb}
\usepackage{subcaption}
\usepackage{microtype}      % microtypography
\usepackage{xcolor}         % colors
\usepackage{mathtools}
\usepackage{enumitem}
\usepackage{multicol}
\usepackage{multirow}
\usepackage{makecell}
\usepackage{pifont}
\usepackage{booktabs}
\usepackage{makecell}
\newtheorem{theorem}{Theorem}

\newtheorem{proof}{Proof}

\newtheorem{proposition}{Proposition}
\usepackage{newfloat}
\usepackage{listings}
\makeatletter

\newcommand{\Rmnum}[1]{\expandafter\@slowromancap\romannumeral #1@}
\makeatother
\def\BibTeX{{\rm B\kern-.05em{\sc i\kern-.025em b}\kern-.08em
    T\kern-.1667em\lower.7ex\hbox{E}\kern-.125emX}}
\begin{document}

\title{Unbiased Image Synthesis via Manifold Guidance in Diffusion Models
}

\author{\IEEEauthorblockN{Xingzhe Su, Daixi Jia, Fengge Wu, Junsuo Zhao, Changwen Zheng, Wenwen Qiang \IEEEauthorrefmark{1} \thanks{* Corresponding author}}
\IEEEauthorblockA{\textit{Institute of Software, Chinese Academy of Sciences.} \\
\textit{University of Chinese Academy of Sciences.} \\ Beijing, 100190, China \\
\{xingzhe2018, jiadaxi2022, fengge, junsuo, changwen, qiangwenwen\}@iscas.ac.cn}
}

\maketitle

\begin{abstract}
Diffusion Models are a potent class of generative models capable of producing high-quality images. However, they often inadvertently favor certain data attributes, undermining the diversity of generated images. This issue is starkly apparent in skewed datasets like CelebA, where the initial dataset disproportionately favors females over males by 57.9\%, this bias amplified in generated data where female representation outstrips males by 148\%. 
In response, we propose a plug-and-play method named Manifold Guidance Sampling, which is also the first unsupervised method to mitigate bias issue in DDPMs. Leveraging the inherent structure of the data manifold, this method steers the sampling process towards a more uniform distribution, effectively dispersing the clustering of biased data. Without the need for modifying the existing model or additional training, it significantly mitigates data bias and enhances the quality and unbiasedness of the generated images.
\end{abstract}


\section{Introduction}
\label{sec:intro}

%DDPM引用，通用的模型，和DM区分，\textbf{研究的意义}
Diffusion Models, also known as Denoising Diffusion Probabilistic Models (DDPMs) \cite{ho2020denoising}, stand out as a superior class of generative models, capable of synthesizing high-quality samples from intricate data domains. Surpassing Generative Adversarial Networks (GANs) in various benchmarks \cite{dhariwal2021diffusion}, they have catalyzed a new wave of applications in image inpainting \cite{lugmayr2022repaint, chungimproving}, image super-resolution \cite{saharia2022image}, and 3D point cloud generation \cite{luo2021diffusion,nichol2022point}. However, despite these significant advancements, DDPMs remain prone to biases inherent in training data, often resulting in skewed and imbalanced outputs \cite{zhao2017men,wang2019balanced,hall2022systematic}.

%bias定义

The manifestation of biases as uneven distributions of data attributes is a growing concern and adversely affects model behavior \cite{wilson2019predictive,buolamwini2018gender,hirota2022quantifying}. This vulnerability to bias extends to DDPMs. A case in point is the CelebA dataset, widely used in facial recognition tasks, which exhibits gender imbalance that female representation surpasses males by 57.9\%. A DDPM trained on this dataset further magnifies this discrepancy, resulting in an output where the proportion of females is 148\% higher than that of males, as depicted in Fig.\ref{fig1}(b) (\textbf{middle}).
\begin{figure}[!h]
\centering
\subfigure[]{\includegraphics[width=0.46\textwidth]{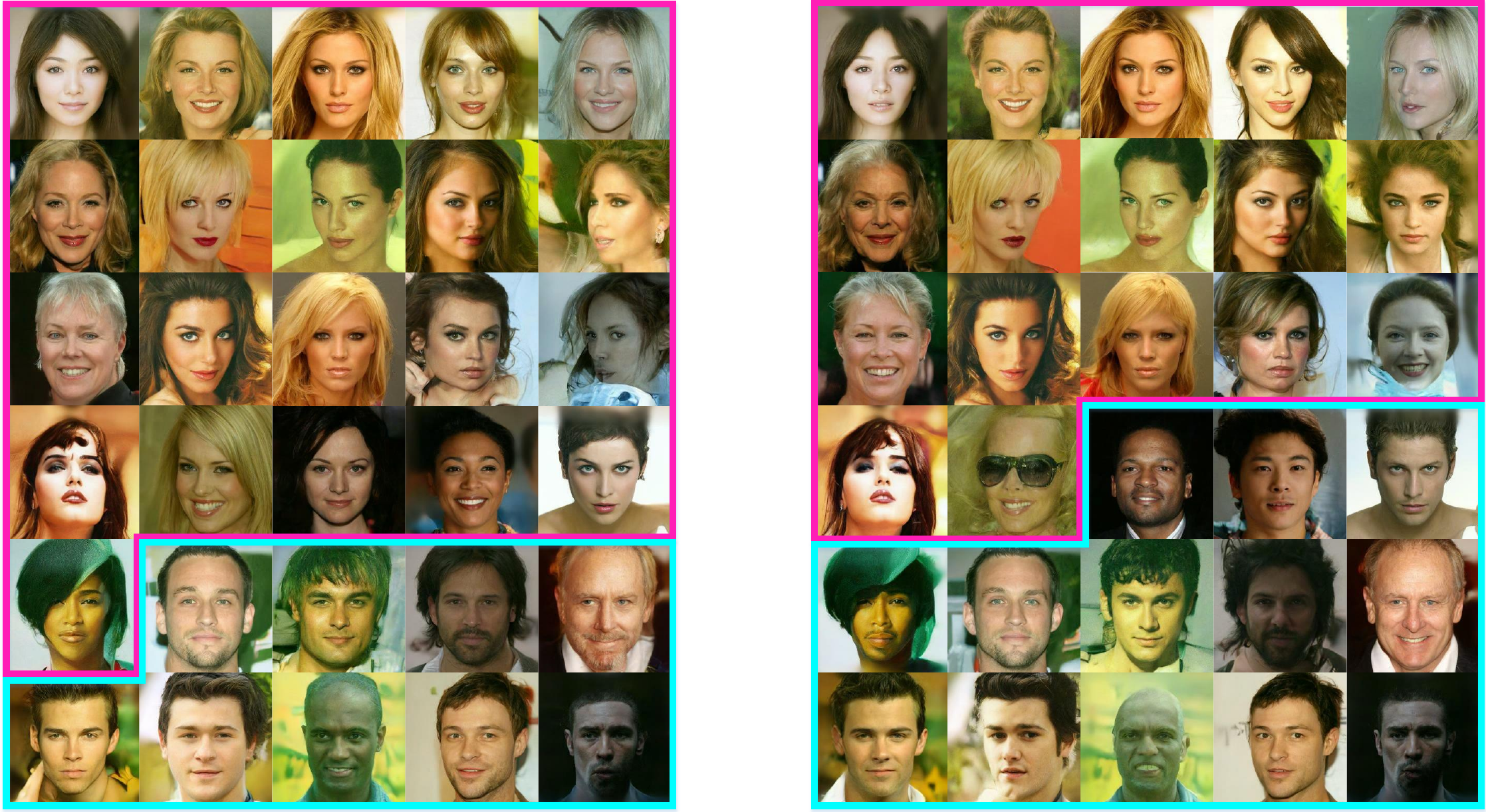}}
\subfigure[]{\includegraphics[width=0.5\textwidth]{pics/pie_gender.pdf}}
\caption{(a) Images generated by DDPM using standard sampling (\textbf{left}), and manifold-guided sampling (Ours) (\textbf{right}), from the same model. (b) Gender distribution in the CelebA dataset (\textbf{left}), generated data from DDPM using standard sampling (\textbf{middle}), and manifold-guided sampling (Ours) (\textbf{right}).}
\label{fig1}
\end{figure}
%An illustrative example is when a model, trained on a dataset that prominently associates women with kitchen objects, reinforces this association in its predictions \cite{zhao2017men}.
%The generated samples of DDPMs often tend to cluster around specific modes of the training data, resulting in non-uniform distribution across the data manifold. This issue becomes more pronounced when the training data itself is biased. In such scenarios, the generative model learns a skewed or imbalanced representation of the data manifold, which can lead to biased generation results and affect downstream applications.

%This skewed representation can adversely impact tasks that rely on DDPMs, such as label-efficient semantic segmentation \cite{baranchuk2021label,brempong2022denoising}, unsupervised representation learning \cite{ddae2023,yang2023diffusion} and estimation of statistical quantities of the data geometry \cite{batzolis2022your}. Given the unsupervised nature of these tasks, identifying and mitigating biases preemptively is challenging. While some studies mitigate bias issue in text-to-image generation models \cite{orgad2023editing, friedrich2023fair,seshadri2023bias}, a gap persists in unsupervised applications. 
%Lacking of explicit guidance on the generation distribution leads to DDPMs inheriting or even exacerbating the bias in the training data. 
%In real-world scenarios, achieving a uniformly distributed dataset is challenging \cite{wang2018devil}, thus complicating efforts to address bias directly from the data perspective. 

Influenced by biased training data, DDPMs often inherit or exacerbate biases in generated images due to a lack of explicit constraints on the generation distribution. Addressing this issue directly from the data perspective typically necessitates retraining the DDPMs, incurring substantial computational and time costs \cite{yang2022diffusion}. While existing text-to-image generation studies have attempted bias mitigation without retraining \cite{orgad2023editing, friedrich2023fair,seshadri2023bias}, these methods generally depend on prior bias knowledge, which is not readily available preemptively.

To overcome these challenges, we introduce Manifold Guidance Sampling ($MGS$), an innovative unsupervised strategy, leveraging the common assumption that the true data distribution is uniformly spread across its intrinsic manifold \cite{brown2022verifying,liu2021rectifying}. $MGS$ modifies the DDPMs’ sampling process without additional training, aligning the generated data with the true data manifold. This alignment disrupts the biased clustering of data, leading to a fairer distribution of generated samples and enhanced diversity across image classes, all without the need for prior bias knowledge. To the best of our knowledge, $MGS$ is the first unsupervised method designed to systematically counteract bias in DDPMs. %方法的亮点，无监督

%方法核心，具体切入，只介绍方法
Our method consists of two main components: (1) an unsupervised technique to estimate the true data manifold from training data, designed to disregard the influence of biased samples, and (2) a Manifold Guidance Sampling mechanism to guide the sampling process of DDPMs towards uniform distribution on the estimated manifold. Notably, our method does not require any labels or retraining of the diffusion model, and can be applied to any pre-trained model. We provide theoretical analysis and empirical evidence to show that our method can improve the quality and unbiasedness of samples. %The main contributions are: 

%\begin{enumerate}
%    \item We propose an unsupervised method to estimate the data manifold from the training data, which captures the intrinsic structure and diversity of the data domain.
%    \item We introduce Manifold Constraint Guidance ($MCG$) to guide the sampling process of diffusion models towards uniform distribution on the estimated manifold, which mitigates data bias and enhances sample quality.
%    \item We empirically demonstrate the effectiveness and versatility of our method through extensive experiments on various diffusion models and datasets.
%\end{enumerate}
\section{Preliminary DDPMs}
\label{bg}
Let $\boldsymbol{x}_0$ be an image, DDPMs firstly construct a forward process $p(\boldsymbol{x}_{1:N}|\boldsymbol{x}_0)$ that injects noise to the data distribution $p(\boldsymbol{x}_0)$. This diffusion process, a type of Markov chain, adheres to $p\left(\boldsymbol{x}_{t} \mid \boldsymbol{x}_{t-1}\right) = \mathcal{N}\left(\boldsymbol{x}_{t} \mid \sqrt{\alpha_{t}} \boldsymbol{x}_{t-1}, \beta_{t} \boldsymbol{I}\right)$,where $\boldsymbol{I}$ is the identity matrix, $\beta_t \in (0,1)$ is the forward noise schedule, and $\alpha_t \coloneqq 1-\beta_t, t = 1,\dots,N$. Utilizing the properties of the conditional Gaussian distribution, the transition probabilities can be expressed as:
\begin{eqnarray}
\setlength{\abovedisplayskip}{3pt}
\setlength{\belowdisplayskip}{3pt}
p\left(\boldsymbol{x}_{t} \mid \boldsymbol{x}_{0}\right) & = & \mathcal{N}\left(\sqrt{\bar{\alpha}_{t}} \boldsymbol{x}_{0},\left(1-\bar{\alpha}_{t}\right) \boldsymbol{I}\right) \\ \nonumber
p\left(\boldsymbol{x}_{t-1} \mid \boldsymbol{x}_{t}, \boldsymbol{x}_{0}\right) & = & \mathcal{N}\left(\bar{\mu}_{t}\left(\boldsymbol{x}_{t}, \boldsymbol{x}_{0}\right), \bar{\beta}_{t} \boldsymbol{I}\right)
\label{eq1}
\end{eqnarray}
Here, $\bar{\alpha}_{t} = \prod_{i= 1}^{t} \alpha_{i}$, $\bar{\mu}_{t} = \frac{\sqrt{\bar{\alpha}_{t-1}} \beta_{t}}{1-\bar{\alpha}_{t}} x_{0}+\frac{\sqrt{\alpha_{t}}\left(1-\bar{\alpha}_{t-1}\right)}{1-\bar{\alpha}_{t}} x_{t}$ and $\bar{\beta}_{t}= \frac{1-\bar{\alpha}_{t-1}}{1-\bar{\alpha}_{t}} \beta_{t}$.

The reverse process is similarly a Markov chain, designed to approximate $p(\boldsymbol{x}_0)$ by progressively removing noise, starting from a Gaussian distribution $q(\boldsymbol{x}_N) =\mathcal{N}(\boldsymbol{x}_N \mid 0, \boldsymbol{I})$:
\begin{eqnarray}
q\left(\boldsymbol{x}_{0: N}\right) =  q\left(\boldsymbol{x}_{N}\right) \prod_{t = 1}^{N} q\left(\boldsymbol{x}_{t-1} \mid \boldsymbol{x}_{t}\right), \\ q\left(\boldsymbol{x}_{t-1} \mid \boldsymbol{x}_{t}\right) =  \mathcal{N}\left(\boldsymbol{x}_{t-1} \mid \boldsymbol{\mu}_{t}\left(\boldsymbol{x}_{t}\right), \bar{\beta}_{t} \boldsymbol{I}\right)
\end{eqnarray}
$\boldsymbol{\mu}_{t}\left(\boldsymbol{x}_{t}\right)$ is generally parameterized by a time-dependent noise estimation model $\epsilon_{\theta}(\boldsymbol{x}_t,t)$ \cite{ho2020denoising}. The optimization function is:
\begin{eqnarray}
L_{t-1}  =  \mathbb{E}_{\boldsymbol{x}_{0}, \epsilon}\left[\eta \left\|\epsilon-\epsilon_{\theta}\left(\sqrt{\bar{\alpha}_{t}} \boldsymbol{x}_{0}+\sqrt{1-\bar{\alpha}_{t}} \epsilon, t\right)\right\|^{2}\right],
\end{eqnarray}
where $\epsilon_{\theta}$ is an estimation of the noise $\epsilon$, $\eta = \frac{\beta_{t}^{2}}{\alpha_{t}\left(1-\bar{\alpha}_{t}\right)}$ \cite{songdenoising}.

The diffusion process can also be interpreted as solving a certain stochastic differential equation (SDE) \cite{song2020score}: $\mathrm{d} \mathbf{x} = \mathbf{f}(\mathbf{x}, t) \mathrm{d} t+g(t) \mathrm{d} \mathbf{w}$, where $\mathbf{f}(\mathbf{x}, t)$ and $g(t)$ are diffusion and drift functions of the SDE, and $w$ is a standard Wiener process. Intriguingly, any diffusion process described by the SDE can be reversed through a corresponding reverse-time SDE:
\begin{eqnarray}
\mathrm{d} \mathbf{x}  =  \left[\mathbf{f}(\mathbf{x}, t)-g(t)^{2} \nabla_{\mathbf{x}} \log p_{t}(\mathbf{x})\right] \mathrm{d} t+g(t) \mathrm{d} \overline{\mathbf{w}}
\label{reverse-sde}
\end{eqnarray}
Here $\overline{\mathbf{w}}$ is a standard Wiener process when time flows backwards, and $\mathrm{d} t$ denotes an infinitesimal negative time step.

\begin{figure}[htbp]
\centering
\includegraphics[width=0.5\textwidth]{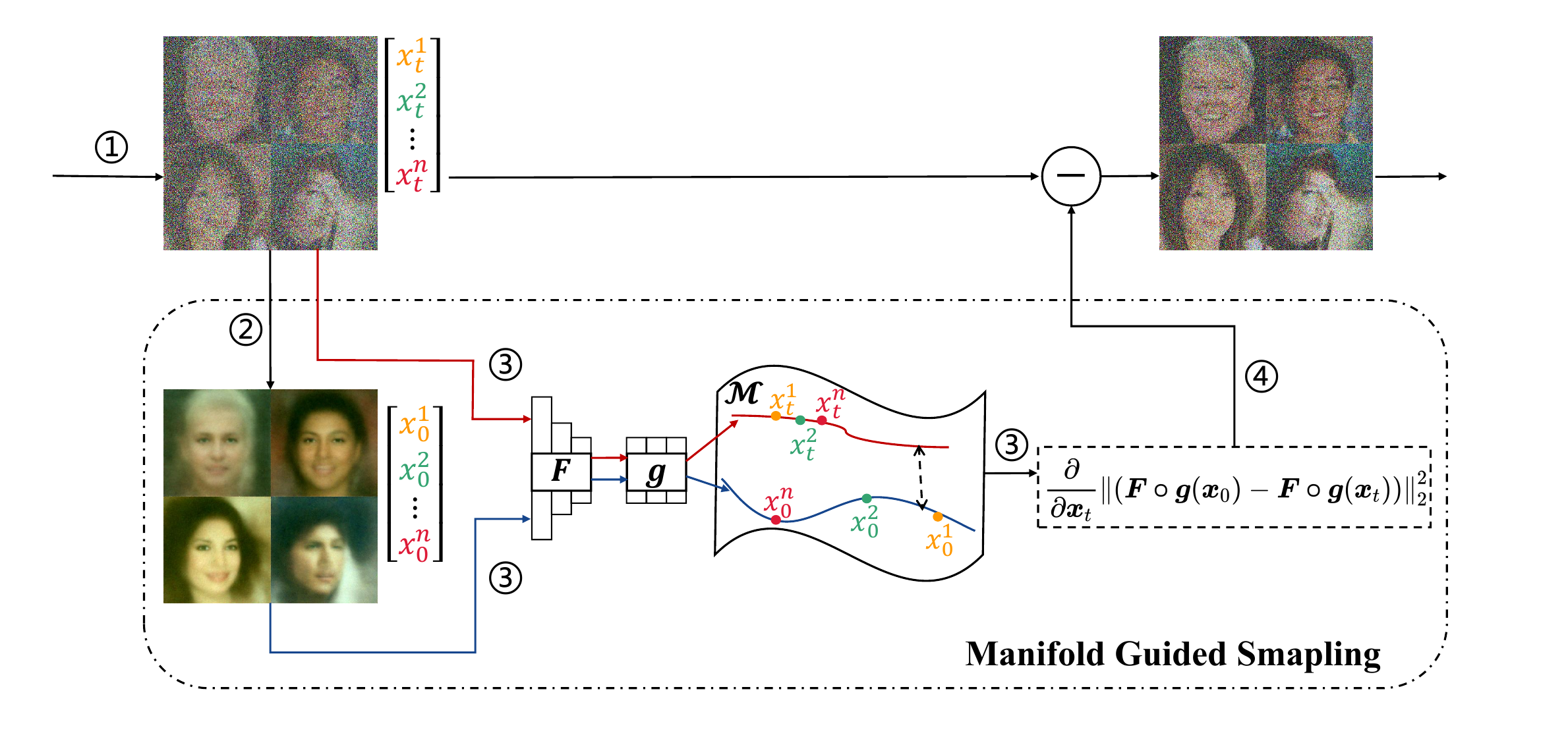}
\caption{The sampling process of DDPMs at step $t$. \ding{172} Unconditional reverse diffusion generates $\boldsymbol{x}_t^i, i=1,2,...,n$. \ding{173} Generate $\hat{\boldsymbol{x}}_0^i$ from $\boldsymbol{x}_t^i$ based on Tweedie's formula. \ding{174} Calculate the manifold constraint gradient, where $\boldsymbol{F}$ and $\boldsymbol{g}$ are pre-trained networks. \ding{175} Apply the manifold constraint guidance on the sampling process of DDPMs.}
\label{fig2}
\end{figure}

\section{Method}
\label{method}
Without explicit constraints on the generation distribution, DDPMs tend to be influenced by dataset biases and producing data clustered around certain attributes. In response, we propose Manifold Guidance Sampling ($MGS$) which incorporates manifold constraints into the DDPMs’ sampling process. This method promotes a uniform distribution of generated samples across the data manifold and effectively disperses biased clustering, all without retraining of the DDPMs. $MGS$ unfolds in two phases. First, in \textbf{Section} \ref{3.1}, we propose a novel approach to evaluate the data manifold. Following this, in \textbf{Section} \ref{3.2}, we design a manifold-guided sampling technique in DDPMs. The framework of $MGS$ is shown in Fig.\ref{fig2}.

\subsection{Data Manifold Evaluation}
\label{3.1}

Evaluating the geometric properties of distributions in high-dimensional ambient space, $\mathbb{R}^{D}$, is impractical \cite{ray2021various}. In response, we propose mapping high-dimensional data to a low-dimensional space, thereby reducing data complexity and enabling more accurate estimation of the data manifold. Our evaluation emphasizes two steps: learning an efficient mapping to the low-dimensional space and assessing the intrinsic data manifold within the transformed domain. 

\textbf{Step} \Rmnum{1}\textbf{. Learn an efficient mapping }

In this paper, we make a foundational assumption that the true data manifold comprises of low-dimensional nonlinear submanifolds $\cup_{j = 1}^{k} \mathcal{M}_{j} \subset \mathbb{R}^{D}$, where each submanifold $\mathcal{M}_j$ has a dimension $d_j<<D$. This assumption, supported by existing literature \cite{brown2022verifying}, posits that images from each submanifold $\mathcal{M}_{j} \subset \mathbb{R}^{D}$ can be effectively represented in a lower-dimensional feature space $\mathcal{S}_{j} \subset \mathbb{R}^{d}$. Accordingly, in this optimal feature space which mirrors the structure of true data manifold, image representations should exhibit: \textit{1) Between-Mode Discrepancy. Representations from different submanifolds should be highly uncorrelated; 2) Within-Mode Similarity. Representations from the same submanifold should be relatively correlated; 3) Maximally Variance. Representations should have as large dimension as possible to cover all the submanifolds and be variant for the same submanifold.}

\textbf{1) Objective function}. Our objective is to ensure mapped features exhibit the properties above and capture the complex structure of the data manifold. To ensure between-mode discrepancy, features from different submanifolds should span the largest possible volume for a sparse feature distribution. Conversely, features within the same submanifold should span a minimal volume for a compact feature distribution. Thus, we need to measure the relationship between features, focusing on the feature distribution compactness.

Recall that singular values $\lambda_{i}$ often correspond to important information implied in a matrix $\mathbf{M}$. Singular vectors corresponding to larger singular values delineate the principal stretching directions of the data. Specifically, larger and more numerous singular values are characteristic of more uniformly distributed data, while smaller and fewer singular values indicate compactness. Thus, we can use $\mathcal{S} = \sum_{i = 1} \lambda_{i}^{2}$ to measure the compactness of the feature distribution, with larger value of $\mathcal{S}$ indicating a broader span of the singular vectors and thus greater uniformity of the data. Considering computational efficiency, we opt for the trace of $\mathbf{M} \mathbf{M}^{T}$ instead of $\mathcal{S}$. Our designed objective function is given by: 
\begin{small}
\begin{eqnarray}
\label{eq4}
\mathcal{L}_{M}(Z) & = & \frac{1}{2 n}\left(\operatorname{Tr}\left(Z Z^{\mathrm{T}}\right)-\sum_{j=1}^{k} \operatorname{Tr}\left(Z C^j Z^{\mathrm{T}}\right)\right),
\end{eqnarray}
\end{small}
where $Z=\left[z_{1}, \ldots, z_{n}\right] \subset \mathbb{R}^{d \times n}$ denotes the representations of images, $\mathbf{C}=\left\{C^{j} \in \mathbb{R}^{n \times n}\right\}_{j=1}^{k}$ is a set of positive diagonal matrices whose diagonal entries denote the membership of the $n$ samples in the $k$ submanifolds. %In the supervised setting, $k$ is the number of submanifolds. 
If the sample $x_i$ belongs to the submanifold $j$, then $C^{j}(i,i) = 1$. Otherwise, $C^{j}(i,i) = 0$. 

By maximizing the function $\mathcal{L}_{M}$, we can enforce compact distribution within submanifolds and uniform distribution across all images, thereby satisfying the required properties. Employing $\mathcal{L}_{M}$ as the objective function, a convolutional neural network $\boldsymbol{F}$ can be trained as the efficient mapping. 

\textbf{2) Unsupervised learning.} In the function $\mathcal{L}_{M}$, matrices $\mathbf{C}$ are unknown due to the unavailability of explicit submanifolds. Thus, we provide an unsupervised learning paradigm to infer $\mathbf{C}$. We utilize a three-layer MLP network $\boldsymbol{g}$ to learn the relationships among samples. The network takes the representations $Z$, and outputs a matrix $\mathbf{R} \in \mathbb{R}^{n \times n}$, where $\mathbf{R}(i,j)$ denotes the relationship between sample $x_i$ and $x_j$ for $i,j \in \left\{ {1,...,n} \right\}$. $\mathbf{R}(i,j) = 1$ when $x_i$ and $x_j$ belong to the same submanifold. For each sample $x_j$, we formulate a diagonal matrix $C^{j}$ with $C^{j}(i,i) = \mathbf{R}(j,i)$, thus constructing $\mathbf{C}$ as a set of relationship matrices for each sample.

To facilitate the training of network $\boldsymbol{g}$, we use a pre-trained encoder $\boldsymbol{E}_{pre}$ to obtain the prior feature representations $\bar Z = \left\{ {\bar z_1,...,\bar z_{n}} \right\}$, with $\bar z_i = {\boldsymbol{E}_{pre}}\left( {x_i} \right)$, $i \in \left\{ {1,...,n} \right\}$. Anchoring on each ${{{\bar z}_i}}$ from $\bar Z$, we define a prior relational matrix $\mathbf{R}^{pre}$ using the Gaussian kernel based on the Euclidean distance between the feature representation, modulated by a temperature hyperparameter $\tau$:
\begin{equation}
\mathbf{R}^{pre}_{i,j} = \exp ( {{{-\| {{{\bar z}_i} - {{\bar z}_j}} \|_2^2} \mathord{\left/
 {\vphantom {{\left\| {{{\bar z}_i} - {{\bar z}_j}} \right\|_2^2} \tau }} \right.
 \kern-\nulldelimiterspace} \tau }} )
\end{equation}
%where ${{\bar z}_i}$ is the $i$-th sample in $\bar Z$, and $\tau$ is temperature hyperparameter. 
%Then, we can obtain: $\Delta_i^{pre} = \left[ {\bar \Delta_{i,1}^{pre},...,\bar \Delta_{i,n}^{pre}} \right]$ where $\bar \Delta_{i,k}^{pre} = \Delta_{i,k}^{pre}/\sum\nolimits_{j = 1}^{j = n} {\Delta_{i,j}^{pre}} $ and $k \in \left\{ {1,...,n} \right\}$. We in turn treat the samples in $\bar Z$ as anchors and obtain ${\Delta^{pre}} = {[ {\Delta_1^{pre},...,\Delta_{N}^{pre}} ]^{\rm{T}}}$. To this end, we 
%of $\Delta$ based on a pre-trained encoder and the similarity of different pairs of samples. 

Matrix $\mathbf{R}^{pre}$ serves as a benchmark for the relationships that network $\boldsymbol{g}$ should learn. Thus, we define the loss function for $\boldsymbol{g}$ as the squared Euclidean distance between $\mathbf{R}^{pre}$ and the network’s output $\boldsymbol{g}(Z)$ in Eq.\ref{dswd}. In the optimal case, $\boldsymbol{g}$ will adeptly learn the underlying relationships between samples, correctly identifying members of the same submanifold.
\begin{equation}\label{dswd}
{\mathcal{L}_{con}} = \left\| {{\mathbf{R}^{pre}} - \boldsymbol{g}(Z)} \right\|_2^2
\end{equation}
 %In our experiments, we find that the choice of the pre-trained encoder has no impact on the final results.

\textbf{3) Training.} In our experimental setup, we initiate by training the network $\boldsymbol{g}$ to grasp the intricate relationships between the samples. Then, with the network $\boldsymbol{g}$ fixed, we proceed to train the network $\boldsymbol{F}$. Finally, we engage in joint training of both $\boldsymbol{F}$ and $\boldsymbol{g}$ using Eq.\ref{eq4} to obtain the optimal mapping $\boldsymbol{F}$. 

\textbf{Step} \Rmnum{2}\textbf{. Assess the intrinsic data manifold}

Directly computing the data manifold is inherently complex; however, our approach encapsulates the manifold’s structure by examining the connections between samples as interpreted by the trained networks $\boldsymbol{F}$ and $\boldsymbol{g}$. Specifically, we represent data manifold using the matrix $\mathbf{R} = \boldsymbol{g}\circ \boldsymbol{F}(X)$, where $\circ$ denotes the composition of the functions represented by the networks, applying $\boldsymbol{F}$ to input data $X$ followed by $\boldsymbol{g}$ to produce the relational matrix $\mathbf{R}$. We provide theoretical validation in \textbf{Section} \ref{ta} that our method accurately captures the data manifold in image space. Further, we offer visualizations of the representations $Z$ and the matrix $\mathbf{R}$ in Appendix C.1. This innovative approach to measuring the data manifold empowers us to enhance the sampling process of diffusion models, ensuring a more representative generation of images.% to encourage the generated images to be uniformly distributed on the data manifold.

\subsection{Manifold Guidance Sampling}
\label{3.2}

The central aim of this paper is to facilitate the generation of images uniformly distributed along the true data manifold, denoted as $\boldsymbol{M}$. First, we need to estimate $\boldsymbol{M}$. According to the Tweedie's formula, for a Gaussian variable $z \sim \mathcal{N}\left(z ; \mu_{z}, \Sigma_{z}\right)$, $\mu_{z}=z+\Sigma_{z} \nabla \log p(z)$. From Eq.1, we can get that $\sqrt{\bar{\alpha}_{t}} \boldsymbol{x}_{0}=\boldsymbol{x}_{t}+\left(1-\bar{\alpha}_{t}\right) \nabla \log p\left(\boldsymbol{x}_{t}\right)$. Once the score function $\nabla_{\mathbf{x}} \log p_{t}(\mathbf{x})$ is estimated, we can estimate $\boldsymbol{\hat{x}}_{0}$ at every time step. Similarly, once the noise estimation model $\epsilon_{\theta}$ is trained, we can get $\boldsymbol{\hat{x}}_{0}={[\boldsymbol{x}_{t}-\sqrt{1-\bar{\alpha}_{t}} \boldsymbol{\epsilon}_{\theta}(\boldsymbol{x}_{t},t)]}/{\sqrt{\bar{\alpha}_{t}}}$. Thus, we can approximate the true data manifold at each time step by a batch of estimated $\boldsymbol{\hat{x}}_{0}$, i.e. $\boldsymbol{M} = \boldsymbol{H}( \frac{\boldsymbol{x}_{t}-\sqrt{1-\bar{\alpha}_{t}} \boldsymbol{\epsilon}_{\theta}(\boldsymbol{x}_{t},t)}{\sqrt{\bar{\alpha}_{t}}})$, where $\boldsymbol{H} = \boldsymbol{g\circ F}$ represents the composition of the pre-trained networks from the previous section, $\boldsymbol{x}_t$ is a batch of samples in our experiments.

Next, we need to add manifold constraint on the sampling process of DDPMs. Based on the Bayes rule $p(x|\boldsymbol{M})=p(\boldsymbol{M}|x)p(x)/p(\boldsymbol{M})$, we derive the gradient of the log probability as $\nabla_x \mathrm{log}p(x|\boldsymbol{M}) = \nabla_x \mathrm{log}p(x)+\nabla_x \mathrm{log}p(\boldsymbol{M}|x)$. To align the generated manifold with the true data manifold, we propose $MGS$ that modifies the reverse SDE function (Eq.\ref{reverse-sde}) as follows, without necessitating labels or model retraining:
\begin{eqnarray}
\begin{split}
   \mathrm{d} \mathbf{x} & =  \left[\mathbf{f}(\mathbf{x}, t)-g(t)^{2} \nabla_{\mathbf{x}} \log p_{t}(\mathbf{x})\right] \mathrm{d} t \\ 
   & \quad - \left[ \lambda \frac{\partial}{\partial \boldsymbol{x}}\left\|\left(\boldsymbol{M}-\boldsymbol{H}( \mathbf{x},t)\right)\right\|_{2}^{2} \right] \mathrm{d} t +g(t) \mathrm{d} \overline{\mathbf{w}}.
\end{split}
\end{eqnarray}
Similarly, we adapt the sampling process of DDPMs to integrate $MGS$ into the iterative denoising steps:
\begin{eqnarray}
\begin{split}
    \boldsymbol{x}_{t-1} & =\frac{1}{\sqrt{\alpha_{t}}}\left(\boldsymbol{x}_{t}-\frac{1-\alpha_{t}}{\sqrt{1-\bar{\alpha}_{t}}} \boldsymbol{\epsilon}_{\theta}\left(\boldsymbol{x}_{t}, t\right)\right) \\
    & \quad - {\lambda} \frac{\partial}{\partial \boldsymbol{x}_t}\left\|\left(\boldsymbol{M}-\boldsymbol{H}( \boldsymbol{x}_t)\right)\right\|_{2}^{2}+\sqrt{\bar{\beta}_{t}} \mathbf{z},
\end{split}
\end{eqnarray}
where $\boldsymbol{M}$ is the estimated true data manifold, $\boldsymbol{x}_t$ is a batch of samples, $\lambda$ is the scale of the manifold guidance. Since diffusion models generate the high-level context in the early stage and fine details in the later stage \cite{choi2022perception}, we apply our method in the early stage to achieve semantic changes. 
%imperceptible The whole algorithm of our method is shown in Algorithm 1. We also illustrate our scheme %visually in Fig. 2 (a). The additional MCG leads to a dramatic performance boost, as can %be seen in Fig. 2 (b). In the following, we study the theoretical properties of the %method.

\subsection{Theoretical Analysis}
\label{ta}

In this section, we prove that our data manifold evaluation approach can accurately estimate the data manifold in the image space. We begin by establishing the theoretical connection between $\mathrm{Tr}(\mathbf{M} \mathbf{M}^T)$ and information theory. Then, we illustrate the properties of the optimal solution of Eq.\ref{eq4}.
\begin{proposition}
Given finite samples $Z$ from a distribution $P(z)$, and $\mathbf{M}=\left[z^{1}, \ldots, z^{n}\right] \subset \mathbb{R}^{d \times n}$, the square of the Frobenius norm or $\mathrm{Tr}(\mathbf{M} \mathbf{M}^T)$ represents the compactness of this distribution.
\end{proposition}
The proof of this proposition is available in the Appendix A. In short words, $\mathrm{Tr}(\mathbf{M} \mathbf{M}^T)$ can be seen as a variant to the rate reduction which has been used to measure the “compactness” of a random distribution. Based on this proposition and the data manifold assumption mentioned above, we can infer that the optimal solution of Eq.\ref{eq4} have following properties:
\begin{theorem}
Suppose $\boldsymbol{Z}^{*}$ is the optimal solution that maximizes the objective function Eq.\ref{eq4}. We have:

- Between-Mode Discrepancy: If the ambient space is adequately large, the subspaces are all orthogonal to each other, i.e.  $\left(\boldsymbol{Z}_{i}^{*}\right)^{\top} \boldsymbol{Z}_{j}^{*}=\mathbf{0}$  for $i \neq j$.

- Maximally Variance: If the coding precision is adequately high, i.e.,  $\epsilon^{4}<\min _{j}\left\{\frac{n_{j}}{n} \frac{d^{2}}{d_{j}^{2}}\right\}$, each subspace achieves its maximal dimension, i.e. $\operatorname{rank}\left(\boldsymbol{Z}_{j}^{*}\right)=d_{j}$ .
\end{theorem}

Therefore, Eq.\ref{eq4} can promote embedding of data into multiple independent subspaces. These properties together ensure that the optimal embedding accurately reflects the data manifold in the feature space, which, in turn, corresponds to the data manifold in the image space. We provide the proof of this theorem in Appendix A.

\section{Experiments}
We conduct unsupervised generation experiments on six datasets: MNIST \cite{lecun1998gradient}, Cifar10 \cite{cifar}, CelebA-HQ \cite{celeba}, LSUN-cat, LSUN-bedroom and LSUN-church \cite{lsun}. We show the quantitative and qualitative results of our method, as well as the comparison with existing methods. The ablation studies of the components in our method are available in Appendix C.2.

\subsection{Experiment Setup}
Since our approach is plug-and-play, we can use pre-trained models from DDPM \cite{ho2020denoising}. The generation model for MNIST is finetuned on pre-trained Cifar10 model. We maintain a total of $N=1000$ steps with a linear variance schedule and employ three popular schedulers for sampling: DDPM \cite{ho2020denoising}, DDIM \cite{songdenoising} and PNDM \cite{liupseudo}, to ensure a comprehensive analysis. Our architecture includes a ResNet-based $\boldsymbol{F}$ and a three-layer MLP $\boldsymbol{g}$. We choose a batch size of 32 for MNIST and Cifar10 datasets, and 16 for other datasets. We opt for SimCLR \cite{simclr}, a self-supervised model \cite{guo2024self,qiang2022interventional}, as $\boldsymbol{E}_{pre}$. Please refer to the Appendix B.2 and C.2 for the rationale for these choices and detailed implementation.
To assess the performance of our method, we employ Frechet Inception Distance (FID) \cite{FID}, the mostly widely-used metric, for evaluating the quality and diversity of the generated images and include sFID \cite{sFID} to capture the spatial distributional similarity between images.
% We also include  as a metric to capture the spatial distributional similarity between images. sFID is equivalent to FID but uses intermediate spatial features in the inception network rather than the spatially-pooled features used in standard FID. 

\begin{figure*}[htbp]
\centering
\includegraphics[width=0.9\textwidth]{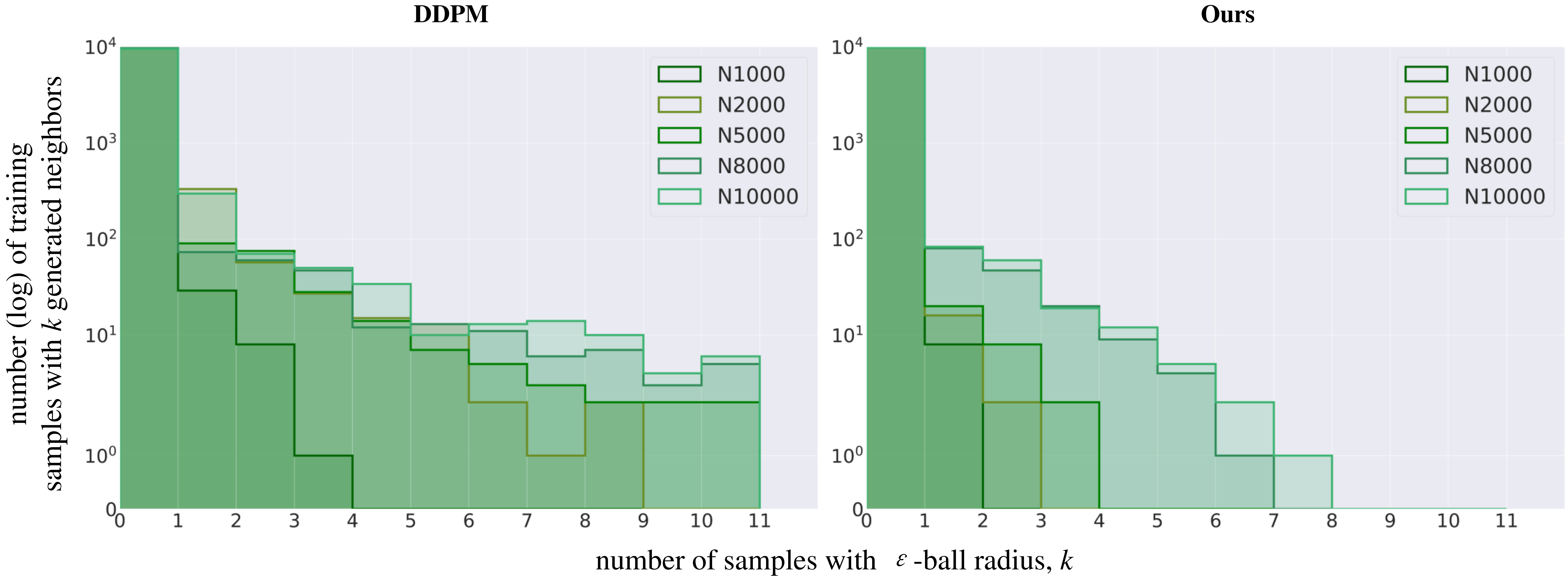}
\caption{Distribution of the number of MNIST samples with $k$ generated neighbors within an $\varepsilon$-ball radius. $N$ samples are generated using DDPM model (\textbf{left}) and our method (\textbf{right}). $\varepsilon$ is the average nearest neighbor distance for the MNIST samples.}
\label{fig5_1}
\end{figure*}

\begin{figure}[!h]
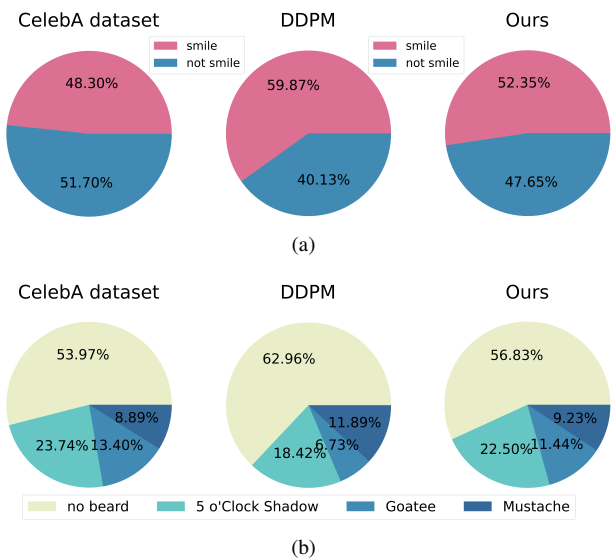

\centering
\subfigure[]{\includegraphics[width=0.46\textwidth]{pics/pie_smile1.pdf}}
\subfigure[]{\includegraphics[width=0.46\textwidth]{pics/pie_beard1.pdf}}
\caption{``Smile'' (a) and ``Beard'' (b) distribution in the CelebA dataset (\textbf{left}), generated data from DDPM using standard sampling (\textbf{middle}), and manifold-guided sampling (\textbf{right}).}
\label{fig_a}
\end{figure}

\begin{table}[htbp]
\centering
\caption{Comparison of FID $\downarrow$ on four diverse datasets (the \textcolor{red}{red} numbers present our improvement)}
\renewcommand\arraystretch{1.5}
\footnotesize
\begin{tabular}{cccccccccc}
\hline
Datasets                                       & Scheduler                                      &50                                         & 100                                       & 1000                                       \\ \hline
\multirow{3}{*}{\textbf{Cifar10}}                       & \textbf{DDPM}                                                           & 32.72\textcolor{red}{-4.22}                                     & 9.99\textcolor{red}{-3.12}                        & 3.17\textcolor{red}{-0.42}             \\
                                               & \textbf{DDIM}                                     &  6.99\textcolor{red}{-2.37}                     &  5.52\textcolor{red}{-2.24}                     &  4.00\textcolor{red}{-1.23}                                           \\
                                               & \textbf{PNDM}                                &  3.95\textcolor{red}{-0.85} & 3.72\textcolor{red}{-0.92} &  3.70\textcolor{red}{-0.78} \\ \hline
\multirow{3}{*}{\textbf{\makecell[c]{LSUN-\\church}}}                   & \textbf{DDPM }                                            & 10.84\textcolor{red}{-2.05}                                         & 10.58\textcolor{red}{-2.37}                                       & 7.89\textcolor{red}{-1.10}                                  \\
                                               & \textbf{DDIM}                           &  10.0\textcolor{red}{-2.58}                     &  9.84\textcolor{red}{-2.05}                     &  9.88\textcolor{red}{-1.12}                     &                      \\
                                               & \textbf{PNDM}                      & 9.89\textcolor{red}{-1.55} & 10.1\textcolor{red}{-1.19} & 10.1\textcolor{red}{-1.22}  \\ \hline
\multirow{3}{*}{\textbf{\makecell[c]{LSUN-\\bedroom}}}                  & \textbf{DDPM }                       & 10.81\textcolor{red}{-2.13}                                      & 6.81\textcolor{red}{-1.66}                           & 6.36\textcolor{red}{-1.23}                                         \\
                                               & \textbf{DDIM}                           &  6.05\textcolor{red}{-2.11}                     &  5.97\textcolor{red}{-1.54}                     &  6.32\textcolor{red}{-1.10}                                         \\
                                               & \textbf{PNDM }                    & 6.44\textcolor{red}{-1.08} & 6.91\textcolor{red}{-1.12}  & 6.92\textcolor{red}{-0.91}  \\ \hline
\multicolumn{1}{c}{\multirow{2}{*}{\textbf{CelebA-HQ}}} & \textbf{DDIM}                        & 8.95\textcolor{red}{-2.20}  & 6.36\textcolor{red}{-1.47}  & 3.41\textcolor{red}{-0.79}  \\
\multicolumn{1}{l}{}                           & \textbf{PNDM }                        & 3.34\textcolor{red}{-1.03}                                        & 2.81\textcolor{red}{-0.56}                                        & 2.86\textcolor{red}{-0.73}                         \\ \hline
\end{tabular}
\label{table1}
\end{table}

\subsection{Experiment Results}

First, we focus on the MNIST dataset, characterized by its relative simplicity and balanced distribution of samples across categories. We randomly select a datum and count the number of generated samples ($k$) that fall within the $\varepsilon$-ball neighborhood of this datum, where $\varepsilon$ represents the average nearest neighbor distance among real samples. This procedure is repeated across 10,000 real samples. If the value of $k$ remains consistent, it strongly suggests a uniform distribution of generated samples on the real data manifold. We perform this experiment on the original DDPM model and our method, with the number of generated samples $N$ ranging from 1,000 to 10,000. The distribution of $k$ is displayed in Fig.\ref{fig5_1}. A uniform distribution of generated samples across the data manifold corresponds to a consistent $k$ value across all real samples, resulting in a Dirac distribution in the histograms presented. It can be observed that our method tends towards such distribution, whereas the baseline model exhibits a heavy-tailed distribution of $k$. Please refer to Appendix C.1 for additional experiment results.

Next, we conduct a quantitative analysis on the CelebA dataset, which provides 202,599 face images with 40 distinct facial attributes. We focus our bias analysis on three attributes: gender, smile, and beard type. For each attribute, we generate 12,000 images, classify them using pre-trained models, and calculate the proportions for each category over 10 iterations to average out the results. It can be observed that the original DDPM model significantly favors certain attributes, such as female faces, as demonstrated in Fig.\ref{fig1}(b). In the ``smile'' attribute, with a more uniform distribution in the CelebA dataset (Fig.\ref{fig_a}(a)), DDPM tends to bias towards smiling faces. Similarly, for beard type (Fig.\ref{fig_a}(b)), DDPM exhibits noticeable preferences, while our method approaches the distribution of the training dataset more closely. While our method doesn't achieve completely unbiased generation, it demonstrates effectiveness in addressing bias compared to DDPM. Experiments on other attributes are provided in the Appendix C.1.

Lastly, we evaluate our method using FID metric across various datasets, time steps, and schedulers, as shown in Table \ref{table1}. The sFID results are available in the Appendix C.1. Notably, the FID scores improve with the application of our method, as indicated by the red numbers in the table. This improvement is consistent across different conditions, showcasing the effectiveness and adaptability of our method. Not requiring model retraining or labels, our approach presents a versatile and effective solution for generating unbiased and diverse images. Visual comparisons between images generated by the original DDPM and our method, highlighting the increased diversity and quality of our generated images, are available in Appendix C.3. These experiments collectively underscore the robustness and broad applicability of our approach.

\subsection{Ablation Study}
In this section, we provide ablation and analysis on the hyperparameters. We conduct ablation studies on the hyperparameter $\lambda$. We select the pre-trained DDPM model with DDIM scheduler and set the sampling time steps to 50. Table \ref{table4} presents the FID scores of the ablation study conducted on the Cifar10, CelebA-HQ, LSUN-church, and LSUN-bedroom datasets. We empirically set $\lambda = 1$ on Cifar10 dataset, $\lambda = 5$ on CelebA-HQ and LSUN-church datasets, $\lambda = 10$ on LSUN-bedroom dataset. 

\begin{table}
\centering
\caption{Ablation study on the hyperparameter $\lambda$}
\renewcommand\arraystretch{1.2}
\begin{tabular}{ccccc}
\hline
$\lambda$ & 1    & 5                         & 10   & 20 \\ \hline
\textbf{Cifar10}                & \textbf{4.62} & 5.79                               & 7.22          & 12.3        \\
\textbf{Church}            & 8.25          & \multicolumn{1}{c}{\textbf{7.42}} & 7.97          & 10.21       \\
\textbf{Bedroom}           & 5.33          & \multicolumn{1}{c}{4.69}          & \textbf{3.94} & 7.41        \\
\textbf{CelebA}              & 7.92          & \textbf{6.75}                      & 7.19          & 10.36       \\ \hline
\end{tabular}
\label{table4}
\end{table}

We also conduct ablation studies on the number of guidance steps. The experiment results on the Cifar10 dataset with 50 sampling time steps are shown in Table \ref{table5}. We apply $MGS$ only on the first few sampling steps. It can be seen that different schedulers have different optimal guidance steps. We choose the best guidance steps for each scheduler. More experiments on other datasets are avaliable in the Appendix C.2.

\begin{table}
\centering
\caption{Ablation study on the Cifar10 dataset}
\renewcommand\arraystretch{1.2}
\begin{tabular}{ccccc}
\hline
\textbf{Scheduler} & 5              & 10                         & 20    & 50    \\ \hline
\textbf{DDIM}      & \textbf{4.62}  & 5.33                       & 6.42  & 7.26  \\
\textbf{PNDM}      & \textbf{3.10}  & \multicolumn{1}{c}{3.21}  & 4.03  & 3.92  \\
\textbf{DDPM}      & \textbf{28.50} & \multicolumn{1}{c}{29.69} & 29.77 & 33.14 \\ \hline
\end{tabular}
\label{table5}
\end{table}
%\begin{figure*}[htbp]
%\centering
%\includegraphics[width=0.9\textwidth]{pics/fig5.pdf}
%\caption{Random images generated from DDPM (\textbf{left}, FID=19.75) and $MCG$ (\textbf{right}, FID=16.38) on the LSUN-cat dataset. Images are sorted by the number and the exposed area of cats.}
%\label{fig5}
%\end{figure*}

\section{Conclusion}
In this work, we focus on addressing the bias issue in diffusion models. We proposed $MGS$ to constraint the generated images to be uniformly distributed on the true data manifold. Notably, our method is unsupervised and plug-and-play. This is beneficial since diffusion models have significant computational and energy requirements for training. We provide theoretical and empirical evidence to show the effectiveness and versatility of our method.

\section*{Acknowledgment}
The authors would like to acknowledge the support and the collaboration effort of the project team. This work was supported by Postdoctoral Fellowship Program of CPSF (GZB20230790), China Postdoctoral Science Foundation (2023M743639), CAS Project for Young Scientists in Basic Research, Grant No.YSBR-040, and Special Research Assistant Fund (E3YD590101), Chinese Academy of Sciences.

\bibliographystyle{IEEEbib}
\bibliography{icme2023template}

\newpage
\appendix

The appendix provides additional details and supplementary material to support the main findings and methods presented in this paper. It is organized into several sections: Appendix A provides the proofs of all the theoretical results. Appendix B provides training details for all the experiments that were studied in the main text. Appendix C shows additional results and more visual results that were omitted in the main paper due to page limits. Appendix D introduces the related works and the novelty of our method.

\subsection{Proofs}

In this paper, we assume that the true data manifold comprises of a collection of low-dimensional nonlinear submanifolds $\cup_{j = 1}^{k} \mathcal{M}_{j} \subset \mathbb{R}^{D}$, where each submanifold $\mathcal{M}_j$ has a dimension $d_j<<D$. Accordingly, representations of images in the ideal feature space should have following properties: \textit{1) Between-Mode Discrepancy. Representations from different submanifolds should be highly uncorrelated; 2) Within-Mode Similarity. Representations from the same submanifold should be relatively correlated; 3) Maximally Variance. Representations should have as large dimension as possible to cover all the submanifolds and be variant for the same submanifold. }

We state our proofs below.
\begin{proposition}
Given finite samples $Z$ from a distribution $P(z)$, and $\mathbf{M}=\left[z^{1}, \ldots, z^{n}\right] \subset \mathbb{R}^{d \times n}$, the square of the Frobenius norm or $\mathrm{Tr}(\mathbf{M} \mathbf{M}^T)$ represents the compactness of this distribution.
\end{proposition}
\begin{proof}
\nonumber
In information theory, there are many methods that can be used to measure the “compactness” of a random distribution such as entropy and rate distortion. Among these methods, rate distortion is a more suitable choice than entropy for continuous random variables with degenerate distributions. Given a random variable $z$ and a prescribed precision $\epsilon > 0$, the rate distortion $R(z,\epsilon )$ is the minimal number of binary bits needed to encode $z$ such that the expected decoding error is less than $\epsilon$. Given finite samples $\mathbf{\mathbf{M}}=\left[z^{1}, \ldots, z^{m}\right] \subset \mathbb{R}^{d \times m}$ follow a subspace-like distribution, \cite{ma2007segmentation} provides a precise estimate on the number of binary bits needed to encoded these samples: 
\begin{eqnarray}
    \mathcal{R}(\boldsymbol{\mathbf{M}}, \epsilon) \doteq\left(\frac{m+d}{2}\right) \log \operatorname{det}\left(\boldsymbol{I}+\frac{d}{m \epsilon^{2}} \boldsymbol{\mathbf{M}} \boldsymbol{\mathbf{M}}^{T}\right).
\end{eqnarray}
Based on the first-order Taylor series approximation, $\log \operatorname{det}(\mathbf{C}+\mathbf{D}) \approx \log \operatorname{det}(\mathbf{C})+\operatorname{Tr}\left(\mathbf{D}^{T} \mathbf{C}^{-1}\right)$, we can get:
\begin{equation*}
   \begin{aligned}
    & \quad \frac{1}{2n} \boldsymbol{\operatorname { T r }}(\mathbf{M} \mathbf{M}^{T}) \\
    &= \frac{1}{2} ( \frac{1}{n} \boldsymbol{\operatorname { T r }}(\mathbf{M} \mathbf{M}^{T}) + \log \operatorname{det}(I))\\
    &\approx  \frac{1}{2} \log \operatorname{det}(I+ \frac{1}{n} \mathbf{M} \mathbf{M}^{T})
   \end{aligned}
\end{equation*}

As the sample size $m$ is large, our approach can be seen as an approximation of the rate distortion, which completes our proof. 

\end{proof}

\begin{theorem}
Suppose $\boldsymbol{Z}^{*}$ is the optimal solution that maximizes the objective function $\mathcal{L}_{M}(\boldsymbol{Z})=\frac{1}{2n} ( \boldsymbol{\operatorname { T r }}(\boldsymbol{Z} \boldsymbol{Z}^{T})- \boldsymbol{\operatorname { T r }}(\boldsymbol{Z} \boldsymbol{C} \boldsymbol{Z}^{T}))$. We have:

- Between-Mode Discrepancy: If the ambient space is adequately large, the subspaces are all orthogonal to each other, i.e.  $\left(\boldsymbol{Z}_{i}^{*}\right)^{\top} \boldsymbol{Z}_{j}^{*}=\mathbf{0}$  for $i \neq j$.

- Maximally Variance: If the coding precision is adequately high, i.e.,  $\epsilon^{4}<\min _{j}\left\{\frac{n_{j}}{n} \frac{d^{2}}{d_{j}^{2}}\right\}$, each subspace achieves its maximal dimension, i.e. $\operatorname{rank}\left(\boldsymbol{Z}_{j}^{*}\right)=d_{j}$ .
\end{theorem}
\begin{proof}
\nonumber
Based on the \textbf{Proposition 1}, the objective function $\mathcal{L}_{M}$ can be rewritten as: 
\begin{equation}
\begin{split}
\mathcal{L}_{M}(\mathbf{Z}) & \approx  \frac{1}{2} \log \operatorname{det}\left(\boldsymbol{I}+\alpha \mathbf{Z} \mathbf{Z}^{T}\right) \\
& \quad -\sum_{j=1}^{k} \frac{\gamma_{j}}{2} \log \operatorname{det}\left(\boldsymbol{I}+\alpha_{j} \mathbf{Z} \boldsymbol{C}^{j} \mathbf{Z}^{T}\right)
\end{split}
\end{equation}
where $\alpha=\frac{d}{n \epsilon^{2}}, \alpha_{j}=\frac{d}{\operatorname{tr}\left(\boldsymbol{\Pi}^{j}\right) \epsilon^{2}}, \gamma_{j}=\frac{\operatorname{tr}\left(\boldsymbol{C}^{j}\right)}{n} \text { for } j=1, \ldots, k$. 

Let $\boldsymbol{Z}^{*}=\left[\boldsymbol{Z}_{1}^{*}, \ldots, \boldsymbol{Z}_{k}^{*}\right]$ be the optimal solution of $\mathcal{L}_{M}$.

Suppose that $\left(\boldsymbol{Z}_{j_{1}}^{*}\right)^{\top} \boldsymbol{Z}_{j_{2}}^{*} \neq \mathbf{0}$ for some $1 \leq j_{1}<j_{2} \leq k$. For the optimal solution $\boldsymbol{Z}^{*}$, there holds a strict inequality function. That is,
\begin{eqnarray}
\mathcal{L}_{M}\left(\boldsymbol{Z}^{*}\right)<\sum_{j=1}^{k} \frac{1}{2 n} \log \left(\frac{\operatorname{det}^{n}\left(\boldsymbol{I}+\frac{d}{n \epsilon^{2}} \boldsymbol{Z}_{j}^{*}\left(\boldsymbol{Z}_{j}^{*}\right)^{\top}\right)}{\operatorname{det}^{n_{j}}\left(\boldsymbol{I}+\frac{d}{n_{j} \epsilon^{2}} \boldsymbol{Z}_{j}^{*}\left(\boldsymbol{Z}_{j}^{*}\right)^{\top}\right)}\right)
\label{eq21}
\end{eqnarray}

Since $\sum_{j=1}^{k} d_{j} \leq d$, there exists $\left\{\boldsymbol{U}_{j}^{\prime} \in \mathbb{R}^{d \times d_{j}}\right\}_{j=1}^{k}$  such that the columns of the matrix $\left[\boldsymbol{U}_{1}^{\prime}, \ldots, \boldsymbol{U}_{k}^{\prime}\right]$ are orthonormal. Denote $\boldsymbol{Z}_{j}^{*}=\boldsymbol{U}_{j}^{*} \boldsymbol{\Sigma}_{j}^{*}\left(\boldsymbol{V}_{j}^{*}\right)^{\top}$ the compact SVD of $\boldsymbol{Z}_{j}^{*}$, and let $\boldsymbol{Z}^{\prime}=\left[\boldsymbol{Z}_{1}^{\prime}, \ldots, \boldsymbol{Z}_{k}^{\prime}\right]$, where  $\boldsymbol{Z}_{j}^{\prime}=\boldsymbol{U}_{j}^{\prime} \boldsymbol{\Sigma}_{j}^{*}\left(\boldsymbol{V}_{j}^{*}\right)^{\top}$.

It follows that:

$\left(\boldsymbol{Z}_{j_{1}}^{\prime}\right)^{\top} \boldsymbol{Z}_{j_{2}}^{\prime}=\boldsymbol{V}_{j_{1}}^{*} \boldsymbol{\Sigma}_{j_{1}}^{*}\left(\boldsymbol{U}_{j_{1}}^{\prime}\right)^{\top} \boldsymbol{U}_{j_{2}}^{\prime} \boldsymbol{\Sigma}_{j_{2}}^{*}\left(\boldsymbol{V}_{j_{2}}^{*}\right)^{\top}=\boldsymbol{V}_{j_{1}}^{*} \boldsymbol{\Sigma}_{j_{1}}^{*} \mathbf{0} \boldsymbol{\Sigma}_{j_{2}}^{*}\left(\boldsymbol{V}_{j_{2}}^{*}\right)^{\top}=\mathbf{0}$ for all $1 \leq j_{1}<j_{2} \leq k$.

That is, the matrices $\boldsymbol{Z}_{1}^{\prime}, \ldots, \boldsymbol{Z}_{k}^{\prime}$ are pairwise orthogonal. For $\boldsymbol{Z}^{\prime}$, we can attain that
\begin{eqnarray}
\begin{aligned}
\mathcal{L}_{M}\left(\boldsymbol{Z}^{\prime}\right) & =\sum_{j=1}^{k} \frac{1}{2 n} \log \left(\frac{\operatorname{det}^{n}\left(\boldsymbol{I}+\frac{d}{n \epsilon^{2}} \boldsymbol{Z}_{j}^{\prime}\left(\boldsymbol{Z}_{j}^{\prime}\right)^{\top}\right)}{\operatorname{det}^{n_{j}}\left(\boldsymbol{I}+\frac{d}{n_{j} \epsilon^{2}} \boldsymbol{Z}_{j}^{\prime}\left(\boldsymbol{Z}_{j}^{\prime}\right)^{\top}\right)}\right) \\
& =\sum_{j=1}^{k} \frac{1}{2 n} \log \left(\frac{\operatorname{det}^{n}\left(\boldsymbol{I}+\frac{d}{n \epsilon^{2}} \boldsymbol{Z}_{j}^{*}\left(\boldsymbol{Z}_{j}^{*}\right)^{\top}\right)}{\operatorname{det}^{n_{j}}\left(\boldsymbol{I}+\frac{d}{n_{j} \epsilon^{2}} \boldsymbol{Z}_{j}^{*}\left(\boldsymbol{Z}_{j}^{*}\right)^{\top}\right)}\right)
\label{eq22}
\end{aligned}
\end{eqnarray}
Comparing Eq.\ref{eq21} and Eq.\ref{eq22} gives $\mathcal{L}_{M}\left(\boldsymbol{Z}^{\prime}\right)>\mathcal{L}_{M}\left(\boldsymbol{Z}^{*}\right)$, which contradicts the optimality of  $\boldsymbol{Z}^{*}$. Therefore, we must have
$\left(\boldsymbol{Z}_{j_{1}}^{*}\right)^{\top} \boldsymbol{Z}_{j_{2}}^{*}=\mathbf{0}$ for all $1 \leq j_{1}<j_{2} \leq k$

Moreover, we have
\begin{eqnarray}
\mathcal{L}_{M}\left(\boldsymbol{Z}^{*}\right)=\sum_{j=1}^{k} \frac{1}{2 n} \log \left(\frac{\operatorname{det}^{n}\left(\boldsymbol{I}+\frac{d}{n \epsilon^{2}} \boldsymbol{Z}_{j}^{*}\left(\boldsymbol{Z}_{j}^{*}\right)^{\top}\right)}{\operatorname{det}^{n_{j}}\left(\boldsymbol{I}+\frac{d}{n_{j} \epsilon^{2}} \boldsymbol{Z}_{j}^{*}\left(\boldsymbol{Z}_{j}^{*}\right)^{\top}\right)}\right)
\label{eq23}
\end{eqnarray}
We now prove the result concerning the singular values of $Z_{j}^{*}$. Suppose that there exists $\widetilde{\boldsymbol{Z}}_{j}$ such that $\left\|\widetilde{\boldsymbol{Z}}_{j}\right\|_{F}^{2}=m_{j}, \operatorname{rank}\left(\widetilde{\boldsymbol{Z}}_{j}\right) \leq d_{j}$ and
\begin{eqnarray}
\begin{split}
& \log \left(\frac{\operatorname{det}^{n}\left(\boldsymbol{I}+\frac{d}{n \epsilon^{2}} \widetilde{\boldsymbol{Z}}_{j} \widetilde{\boldsymbol{Z}}_{j}^{\top}\right)}{\operatorname{det}^{n_{j}}\left(\boldsymbol{I}+\frac{d}{n_{j} \epsilon^{2}} \widetilde{\boldsymbol{Z}}_{j} \widetilde{\boldsymbol{Z}}_{j}^{\top}\right)}\right) \\
& >\log \left(\frac{\operatorname{det}^{n}\left(\boldsymbol{I}+\frac{d}{n \epsilon^{2}} \boldsymbol{Z}_{j}^{*}\left(\boldsymbol{Z}_{j}^{*}\right)^{\top}\right)}{\operatorname{det}^{n_{j}}\left(\boldsymbol{I}+\frac{d}{n_{j} \epsilon^{2}} \boldsymbol{Z}_{j}^{*}\left(\boldsymbol{Z}_{j}^{*}\right)^{\top}\right)}\right)
\end{split}
\label{eq24}
\end{eqnarray}

Denote  $\widetilde{\boldsymbol{Z}}_{j}=\widetilde{\boldsymbol{U}}_{j} \widetilde{\boldsymbol{\Sigma}}_{j} \tilde{\boldsymbol{V}}_{j}^{\top}$ the compact  $\mathrm{SVD}$  of $\widetilde{\boldsymbol{Z}}_{j}$ and let $\boldsymbol{Z}^{\prime}=\left[\boldsymbol{Z}_{1}^{*}, \ldots, \boldsymbol{Z}_{j-1}^{*}, \boldsymbol{Z}_{j}^{\prime}, \boldsymbol{Z}_{j+1}^{*}, \ldots, \boldsymbol{Z}_{k}^{*}\right]$, where $ \boldsymbol{Z}_{j}^{\prime}:=\boldsymbol{U}_{j}^{*} \widetilde{\boldsymbol{\Sigma}}_{j} \tilde{\boldsymbol{V}}_{j}^{\top}.$

Note that  $\left\|\boldsymbol{Z}_{j}^{\prime}\right\|_{F}^{2}=m_{j}, \operatorname{rank}\left(\boldsymbol{Z}_{j}^{\prime}\right) \leq d_{j}$ and  $\left(\boldsymbol{Z}_{j}^{\prime}\right)^{\top} \boldsymbol{Z}_{j^{\prime}}^{*}=\mathbf{0}$  for all $j^{\prime} \neq j$. It follows that  $\boldsymbol{Z}^{\prime}$ is a feasible solution to $\mathcal{L}_{M}$ and that the components of $\boldsymbol{Z}^{\prime}$ are pairwise orthogonal. By using Eq.\ref{eq24}, we have
\begin{eqnarray}
\begin{aligned}
  & \quad \mathcal{L}_{M}\left(\boldsymbol{Z}^{\prime}\right) \\
  &=\frac{1}{2 n} \log \left(\frac{\operatorname{det}^{n}\left(\boldsymbol{I}+\frac{d}{n \epsilon^{2}} \boldsymbol{Z}_{j}^{\prime}\left(\boldsymbol{Z}_{j}^{\prime}\right)^{\top}\right)}{\operatorname{det}^{n_{j}}\left(\boldsymbol{I}+\frac{d}{n_{j} \epsilon^{2}} \boldsymbol{Z}_{j}^{\prime}\left(\boldsymbol{Z}_{j}^{\prime}\right)^{\top}\right)}\right)  \\ 
  & \quad +\sum_{j^{\prime} \neq j} \frac{1}{2 n} \log \left(\frac{\operatorname{det}^{n}\left(\boldsymbol{I}+\frac{d}{n \epsilon^{2}} \boldsymbol{Z}_{j^{\prime}}^{*}\left(\boldsymbol{Z}_{j^{\prime}}^{*}\right)^{\top}\right)}{\operatorname{det}^{n_{j^{\prime}}}\left(\boldsymbol{I}+\frac{d}{n_{j^{\prime} \epsilon^{2}}} \boldsymbol{Z}_{j^{\prime}}^{*}\left(\boldsymbol{Z}_{j^{\prime}}^{*}\right)^{\top}\right)}\right)  \\
  &=\frac{1}{2 n} \log \left(\frac{\operatorname{det}^{n}\left(\boldsymbol{I}+\frac{d}{n \epsilon^{2}} \widetilde{\boldsymbol{Z}}_{j}\left(\widetilde{\boldsymbol{Z}}_{j}\right)^{\top}\right)}{\operatorname{det}^{n_{j}}\left(\boldsymbol{I}+\frac{d}{n_{j} \epsilon^{2}} \widetilde{\boldsymbol{Z}}_{j}\left(\widetilde{\boldsymbol{Z}}_{j}\right)^{\top}\right)}\right) \\
  & \quad +\sum_{j^{\prime} \neq j} \frac{1}{2 n} \log \left(\frac{\operatorname{det}^{n}\left(\boldsymbol{I}+\frac{d}{n \epsilon^{2}} \boldsymbol{Z}_{j^{\prime}}^{*}\left(\boldsymbol{Z}_{j^{\prime}}^{*}\right)^{\top}\right)}{\operatorname{det}^{n_{j^{\prime}}}\left(\boldsymbol{I}+\frac{d}{n_{j^{\prime} \epsilon^{2}}} \boldsymbol{Z}_{j^{\prime}}^{*}\left(\boldsymbol{Z}_{j^{\prime}}^{*}\right)^{\top}\right)}\right)  \\
  &>\frac{1}{2 n} \log \left(\frac{\operatorname{det}^{n}\left(\boldsymbol{I}+\frac{d}{n \epsilon^{2}} \boldsymbol{Z}_{j}^{*}\left(\boldsymbol{Z}_{j}^{*}\right)^{\top}\right)}{\operatorname{det}^{n_{j}}\left(\boldsymbol{I}+\frac{d}{n_{j} \epsilon^{2}} \boldsymbol{Z}_{j}^{*}\left(\boldsymbol{Z}_{j}^{*}\right)^{\top}\right)}\right)  \\
  & \quad +\sum_{j^{\prime} \neq j} \frac{1}{2 n} \log \left(\frac{\operatorname{det}^{n}\left(\boldsymbol{I}+\frac{d}{n \epsilon^{2}} \boldsymbol{Z}_{j^{\prime}}^{*}\left(\boldsymbol{Z}_{j^{\prime}}^{*}\right)^{\top}\right)}{\operatorname{det}^{n_{j^{\prime}}}\left(\boldsymbol{I}+\frac{d}{n_{j^{\prime}} \epsilon^{2}} \boldsymbol{Z}_{j^{\prime}}^{*}\left(\boldsymbol{Z}_{j^{\prime}}^{*}\right)^{\top}\right)}\right) \\
  &=\sum_{j=1}^{k} \frac{1}{2 n} \log \left(\frac{\operatorname{det}^{n}\left(\boldsymbol{I}+\frac{d}{n \epsilon^{2}} \boldsymbol{Z}_{j}^{*}\left(\boldsymbol{Z}_{j}^{*}\right)^{\top}\right)}{\operatorname{det}^{n_{j}}\left(\boldsymbol{I}+\frac{d}{n_{j} \epsilon^{2}} \boldsymbol{Z}_{j}^{*}\left(\boldsymbol{Z}_{j}^{*}\right)^{\top}\right)}\right)
\end{aligned}
\end{eqnarray}

Combining it with Eq.\ref{eq23} shows $\mathcal{L}_{M}\left(\boldsymbol{Z}^{\prime}\right)>\mathcal{L}_{M}\left(\boldsymbol{Z}^{*}\right)$, contradicting the optimality of  $\boldsymbol{Z}^{*}$. Therefore, the following result holds:
\begin{eqnarray}
\begin{split}
& \boldsymbol{Z}_{j}^{*} \in \arg \max _{\boldsymbol{Z}_{j}} \log \left(\frac{\operatorname{det}^{n}\left(\boldsymbol{I}+\frac{d}{n \epsilon^{2}} \boldsymbol{Z}_{j} \boldsymbol{Z}_{j}^{\top}\right)}{\operatorname{det}^{n_{j}}\left(\boldsymbol{I}+\frac{d}{n_{j} \epsilon^{2}} \boldsymbol{Z}_{j} \boldsymbol{Z}_{j}^{\top}\right)}\right) \\
& \text { s.t. }\left\|\boldsymbol{Z}_{j}\right\|_{F}^{2}=n_{j}, \operatorname{rank}\left(\boldsymbol{Z}_{j}\right) \leq d_{j}
\end{split}
\label{eq25}
\end{eqnarray}

Observe that the optimization problem in Eq.\ref{eq25} depends on $\boldsymbol{Z}_{j}$  only through its singular values. That is, by letting  $\boldsymbol{\sigma}_{j}:=\left[\sigma_{1, j}, \ldots, \sigma_{\min \left(m_{j}, d\right), j}\right]$ be the singular values of  $\boldsymbol{Z}_{j}$, we have
\begin{eqnarray}
\begin{split}
& \log \left(\frac{\operatorname{det}^{n}\left(\boldsymbol{I}+\frac{d}{n \epsilon^{2}} \boldsymbol{Z}_{j} \boldsymbol{Z}_{j}^{\top}\right)}{\operatorname{det}^{n_{j}}\left(\boldsymbol{I}+\frac{d}{n_{j} \epsilon^{2}} \boldsymbol{Z}_{j} \boldsymbol{Z}_{j}^{\top}\right)}\right) \\
& =\sum_{p=1}^{\min \left\{n_{j}, d\right\}} \log \left(\frac{\left(1+\frac{d}{n \epsilon^{2}} \sigma_{p, j}^{2}\right)^{n}}{\left(1+\frac{d}{n_{j} \epsilon^{2}} \sigma_{p, j}^{2}\right)^{n_{j}}}\right)
\end{split}
\end{eqnarray}
also, we have $\left\|\boldsymbol{Z}_{j}\right\|_{F}^{2}=\sum_{p=1}^{\min \left\{m_{j}, d\right\}} \sigma_{p, j}^{2}$ and $\operatorname{rank}\left(\boldsymbol{Z}_{j}\right)=\left\|\boldsymbol{\sigma}_{j}\right\|_{0}$.

Using these relations, Eq.\ref{eq25} is equivalent to
\begin{eqnarray}
\begin{split}
& \max _{\boldsymbol{\sigma}_{j} \in \mathbb{R}_{+}^{\min \left\{n_{j}, d\right\}}} \sum_{p=1}^{\min \left\{n_{j}, d\right\}} \log \left(\frac{\left(1+\frac{d}{n \epsilon^{2}} \sigma_{p, j}^{2}\right)^{n}}{\left(1+\frac{d}{n_{j} \epsilon^{2}} \sigma_{p, j}^{2}\right)^{n_{j}}}\right) \\
& \text { s.t. } \sum_{p=1}^{\min \left\{n_{j}, d\right\}} \sigma_{p, j}^{2}=n_{j} \text {, and } \operatorname{rank}\left(\boldsymbol{Z}_{j}\right)=\left\|\boldsymbol{\sigma}_{j}\right\|_{0}
\end{split}
\label{eq26}
\end{eqnarray}

Let  $\boldsymbol{\sigma}_{j}^{*}=\left[\sigma_{1, j}^{*}, \ldots, \sigma_{\min \left\{n_{j}, d\right\}, j}^{*}\right]$ be an optimal solution to Eq.\ref{eq26}. We assume that the entries of  $\sigma_{j}^{*}$ are sorted in descending order. It follows that
$\sigma_{p, j}^{*}=0 \text { for all } p>d_{j}$
and
\begin{eqnarray}
\begin{split}
& \left[\sigma_{1, j}^{*}, \ldots, \sigma_{d_{j}, j}^{*}\right] \\
& =\underset{\substack{\left[\sigma_{1, j}, \ldots, \sigma_{d_{j}, j}\right] \in \mathbb{R}_{+}^{d_{j}} \\ \sigma_{1, j} \geq \cdots \geq \sigma_{d_{j}, j}}}{\arg \max } \sum_{p=1}^{d_{j}} \log \left(\frac{\left(1+\frac{d}{n \epsilon^{2}} \sigma_{p, j}^{2}\right)^{n}}{\left(1+\frac{d}{n_{j} \epsilon^{2}} \sigma_{p, j}^{2}\right)^{n_{j}}}\right) \\
& \text { s.t. } \sum_{p=1}^{d_{j}} \sigma_{p, j}^{2}=n_{j}
\end{split}
\label{eq27}
\end{eqnarray}

Then we define
\begin{eqnarray}
f\left(x ; d, \epsilon, n_{j}, n\right)=\log \left(\frac{\left(1+\frac{d}{n \epsilon^{2}} x\right)^{n}}{\left(1+\frac{d}{n_{j} \epsilon^{2}} x\right)^{n_{j}}}\right),
\end{eqnarray}
and rewrite Eq.\ref{eq27} as
\begin{eqnarray}
\max _{\substack{\left[x_{1}, \ldots, x_{d_{j}}\right] \in \mathbb{R}_{+}^{d_{j}} \\ x_{1} \geq \cdots \geq x_{d_{j}}}} \sum_{p=1}^{d_{j}} f\left(x_{p} ; d, \epsilon, n_{j}, n\right) \text { s.t. } \sum_{p=1}^{d_{j}} x_{p}=n_{j} .
\label{eq28}
\end{eqnarray}
We compute the first and second derivative for  $f$ with respect to $x$, which are given by
\begin{eqnarray}
\begin{aligned}
f^{\prime}\left(x ; d, \epsilon, n_{j}, n\right) & =\frac{d^{2} x\left(n-n_{j}\right)}{\left(d x+n \epsilon^{2}\right)\left(d x+n_{j} \epsilon^{2}\right)} \\
f^{\prime \prime}\left(x ; d, \epsilon, n_{j}, n\right) & =\frac{d^{2}\left(n-n_{j}\right)\left(n n_{j} \epsilon^{4}-d^{2} x^{2}\right)}{\left(d x+n \epsilon^{2}\right)^{2}\left(d x+n_{j} \epsilon^{2}\right)^{2}}
\end{aligned}
\end{eqnarray}

Note that: 1) $0=f^{\prime}(0)<f^{\prime}(x)$ for all $x>0$; 2) $f^{\prime}(x)$ is strictly increasing in $\left[0, x_{T}\right]$ and strictly decreasing in $\left[x_{T}, \infty\right)$, where $x_{T}= \epsilon^{2} \sqrt{\frac{n}{d} \frac{n_{j}}{d}}$, and 3) by using the condition $\epsilon^{4}<\frac{n_{j}}{n} \frac{d^{2}}{d_{j}^{2}}$, we have  $f^{\prime \prime}\left(\frac{n_{j}}{d_{j}}\right)<0$.

Therefore, we can conclude that the unique optimal solution to Eq.\ref{eq28} is either $\boldsymbol{x}^{*}=\left[\frac{n_{j}}{d_{j}}, \ldots, \frac{n_{j}}{d_{j}}\right]$, or $\boldsymbol{x}^{*}=\left[x_{H}, \ldots, x_{H}, x_{L}\right]$ for some $x_{H} \in\left(\frac{n_{j}}{d_{j}}, \frac{n_{j}}{d_{j}-1}\right)$ and $x_{L}>0$.

Equivalently, we have either $\left[\sigma_{1, j}^{*}, \ldots, \sigma_{d_{j}, j}^{*}\right]=\left[\sqrt{\frac{n_{j}}{d_{j}}}, \ldots, \sqrt{\frac{n_{j}}{d_{j}}}\right]$, or $\left[\sigma_{1, j}^{*}, \ldots, \sigma_{d_{j}, j}^{*}\right]=\left[\sigma_{H}, \ldots, \sigma_{H}, \sigma_{L}\right]$ for some $\sigma_{H} \in\left(\sqrt{\frac{n_{j}}{d_{j}}}, \sqrt{\frac{n_{j}}{d_{j}-1}}\right)$  and $\sigma_{L}>0$, which complete our proof.

\end{proof}

\subsection{Experimental Details}

\subsubsection{Training Setting}
We mainly use ResNet with 4 residual blocks, which have layer widths [64, 128, 256, 512], in our experiments for network $\boldsymbol{F}$. We use ResNet18 for MNIST and Cifar10, and ResNet50 for other datasets. We use three-layer MLP for network $\boldsymbol{g}$. Across all of our experiments, we use SGD as the optimizer for $\boldsymbol{F}$ and Adam for $\boldsymbol{g}$. We adopt the simple gradient descent–ascent algorithm. We train the models for 800 epochs and use stage-wise learning rate decay every 200 epochs (decay by a factor of 10). We train the models with the full training datasets and on eight NVIDIA RTX A6000 GPUs. 

\subsubsection{Computation Time}
All our sampling steps are performed with a single NVIDIA RTX A6000 GPU. To generate 4 images of size 256 × 256, our method takes about 10 seconds for 25 time steps, 20 seconds for 50 time steps and 300 seconds for 1000 time steps. We conduct ablation study on the batch size of the sampling process. The results are shown in Table \ref{table_b}. In our experiments, given the computational power and time constraints, we have chosen a batch size of 32 for MNIST and Cifar10 datasets, and 16 for other datasets.

\begin{table}
\centering
\caption{Ablation study on the Cifar10 dataset}
\renewcommand\arraystretch{1.2}
\begin{tabular}{cccccc}
\hline
\textbf{Batch} & \textbf{4} & \textbf{8} & \textbf{16} & \textbf{32} & \textbf{64} \\ \hline
\textbf{FID$\downarrow$}   & 4.33       & 4.09       & 3.28        & \textbf{3.14}        & 3.21        \\ \hline
\end{tabular}
\label{table_b}
\end{table}

\subsection{Additional Experimental Results}

\begin{figure*}[h]
\centering
\includegraphics[width=0.9\textwidth]{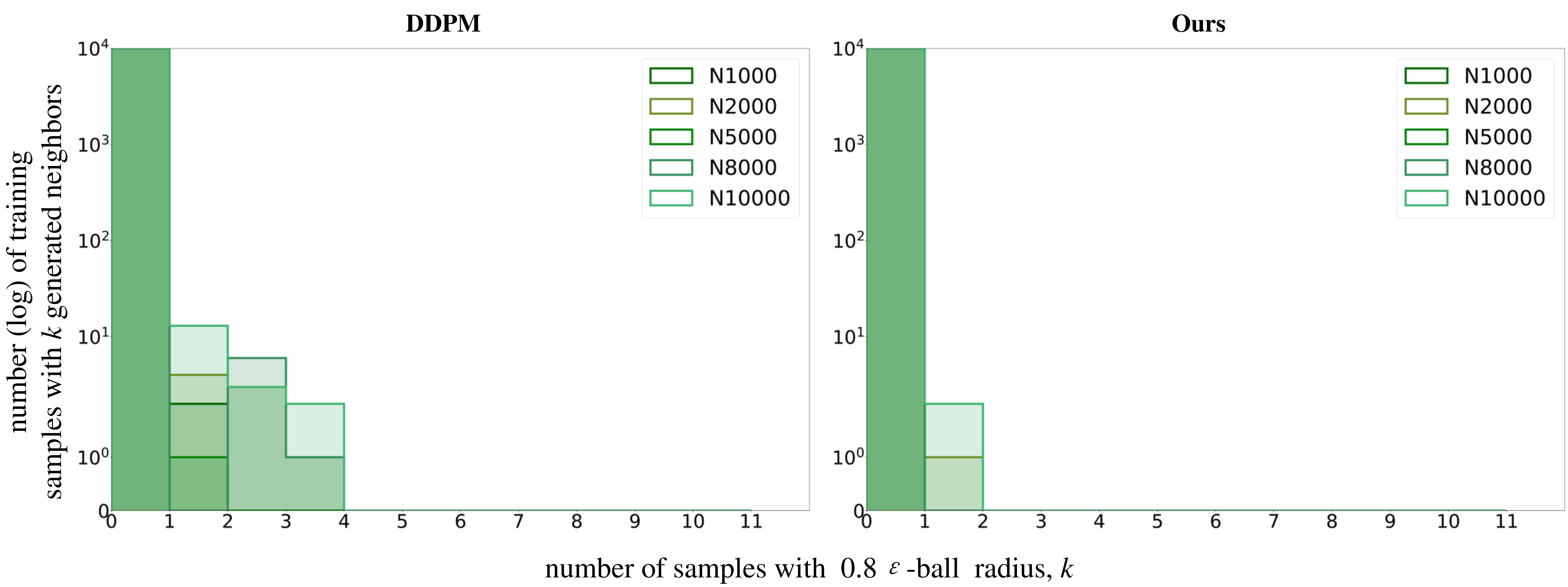}
\caption{Distribution of the number of MNIST training samples with $k$ neighbors generated within a 0.8$\varepsilon$-ball radius. Left shows the results of DDPM. Right shows the results of our method.}
\label{fig3_1}
\end{figure*}

\begin{figure*}[h]
\centering
\includegraphics[width=0.9\textwidth]{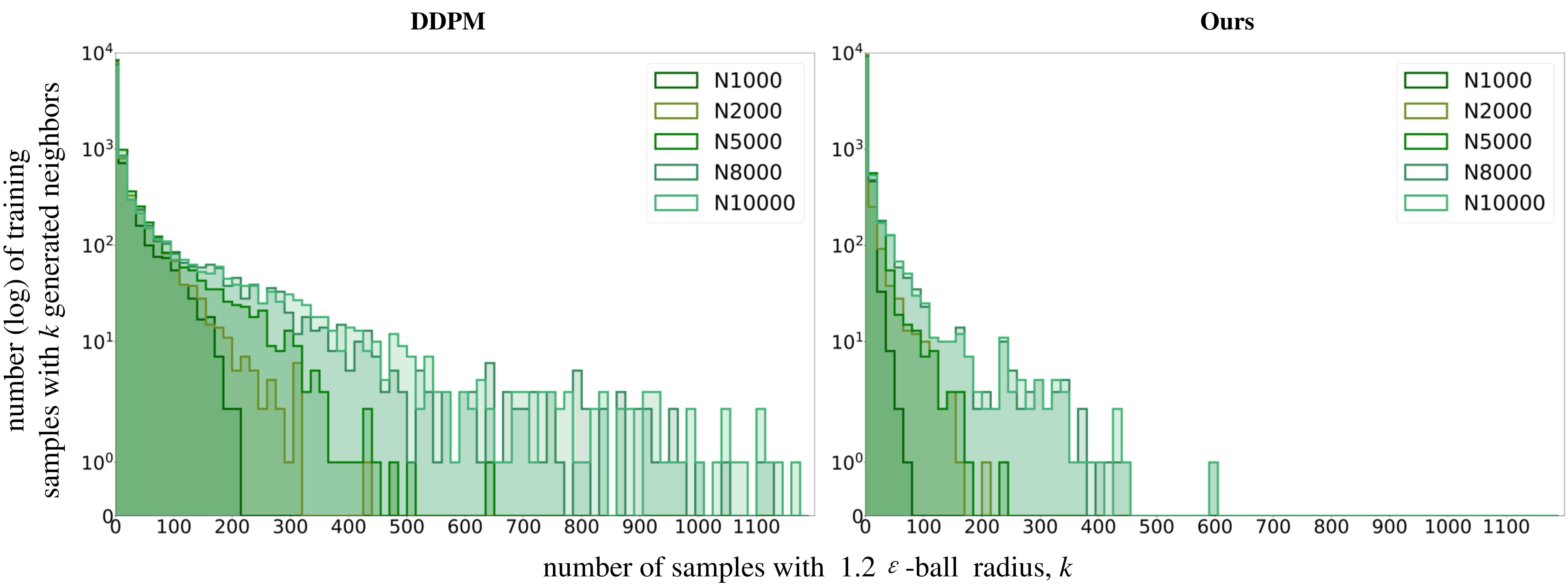}
\caption{Distribution of the number of MNIST training samples with $k$ neighbors generated within a 1.2$\varepsilon$-ball radius. Left shows the results of DDPM. Right shows the results of our method.}
\label{fig3_2}
\end{figure*}

\subsubsection{Additional Results} 
We conduct the same experiment as Fig.3 on the MNIST dataset, with 0.8$\varepsilon$-ball radius and 1.2$\varepsilon$-ball radius. The experiment results are shown in Fig.\ref{fig3_1} and Fig.\ref{fig3_2}.
\begin{figure}[htbp]
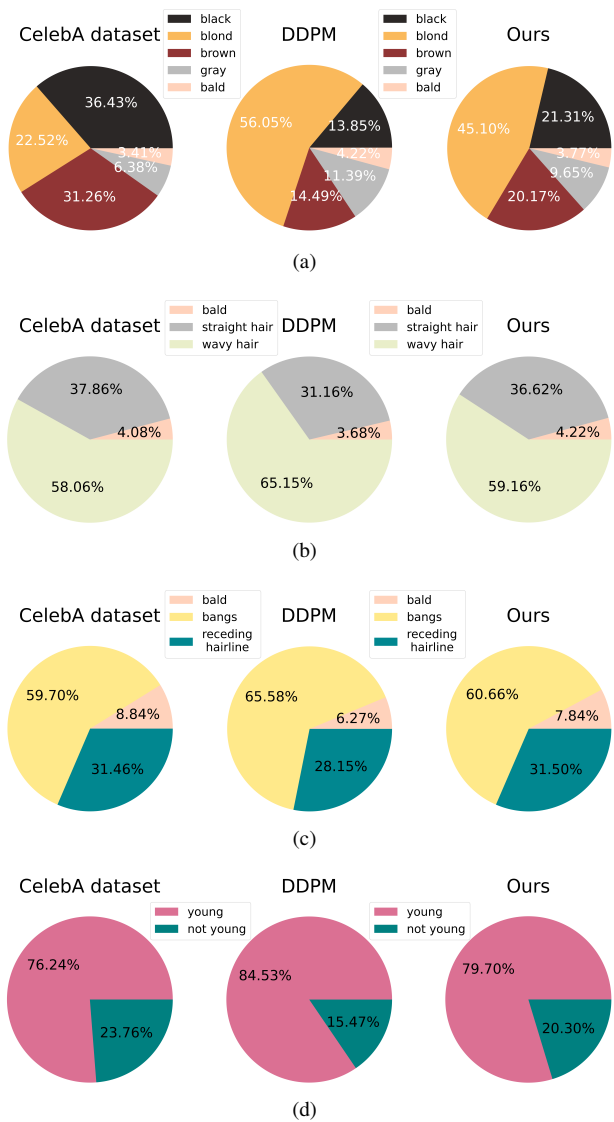

\centering
\subfigure[]{\includegraphics[width=0.46\textwidth]{pics/pie_haircolor1.pdf}}
\subfigure[]{\includegraphics[width=0.46\textwidth]{pics/pie_hairstyle.pdf}}
\subfigure[]{\includegraphics[width=0.46\textwidth]{pics/pie_hairbangs.pdf}}
\subfigure[]{\includegraphics[width=0.46\textwidth]{pics/pie_young.pdf}}
\caption{``Hair Color'' (a), ``Hair Style'' (b), ``Bangs Type'' (c) and ``Young'' (d) distributions in the CelebA dataset (\textbf{left}), generated data from DDPM using standard sampling (\textbf{middle}), and manifold-guided sampling (\textbf{right})}
\label{fig_a1}
\end{figure}

\begin{figure}
    \centering
    \begin{minipage}{0.23\textwidth}
        \centering
        \includegraphics[width=\linewidth]{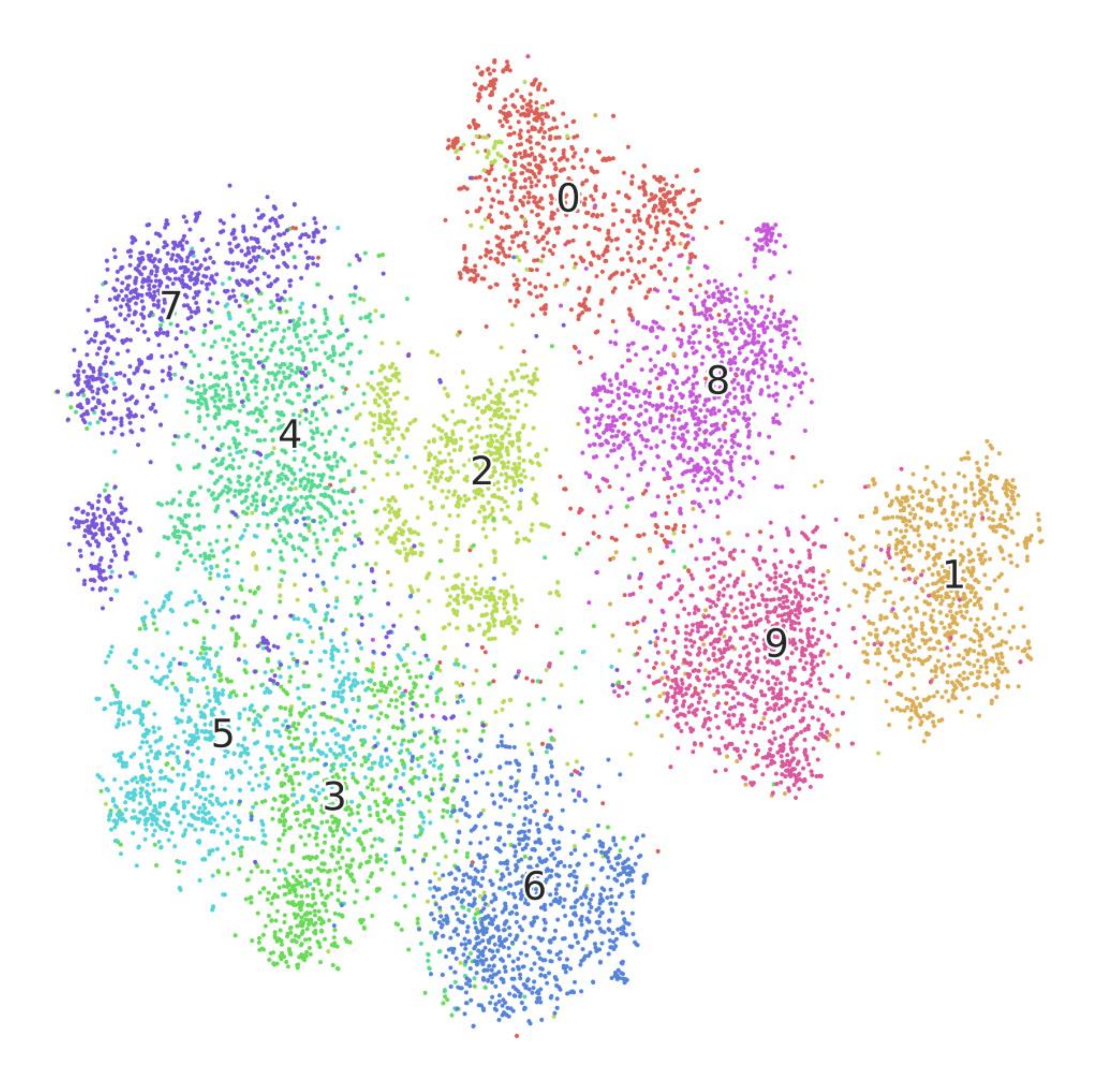}
        \caption*{(a) Training Data}
        \label{fig5:miniImagenet}
    \end{minipage}
    \hfill
    \begin{minipage}{0.23\textwidth}
        \centering
        \includegraphics[width=\linewidth]{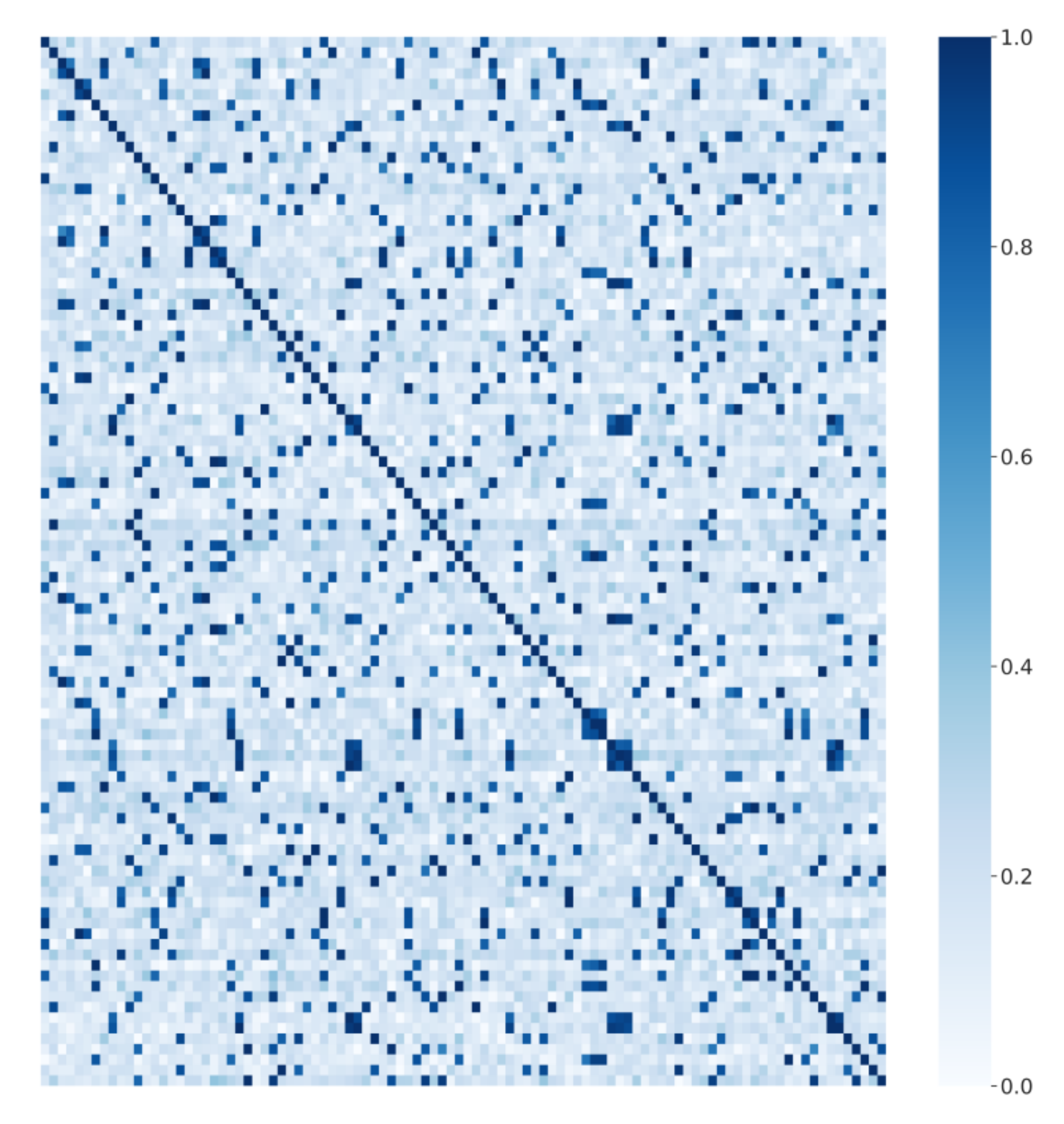}
        \caption*{(b) $\mathbf{R}$ matrix}
        \label{fig15Omniglot}
    \end{minipage}
    
    \begin{minipage}{0.23\textwidth}
        \centering
        \includegraphics[width=\linewidth]{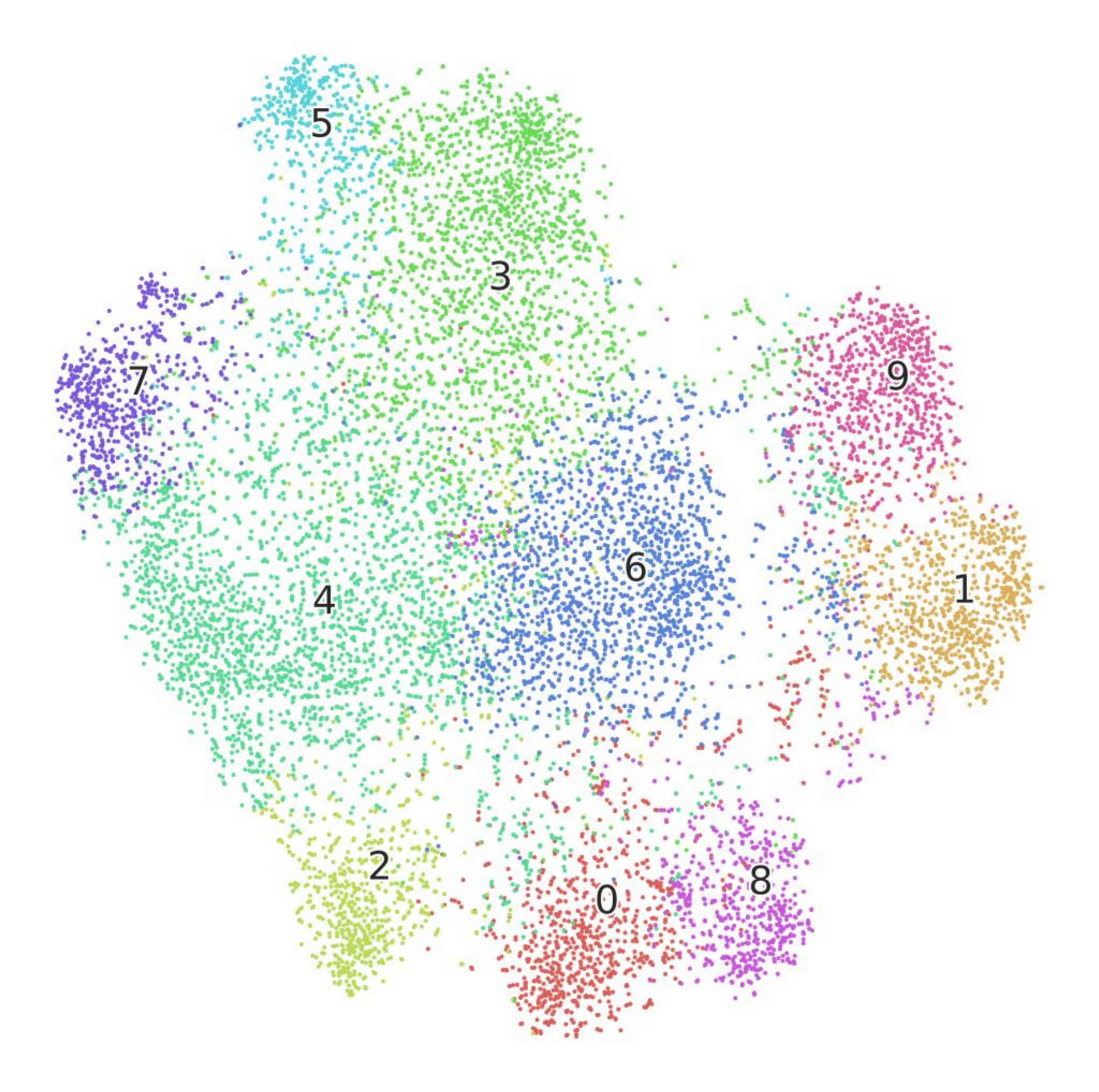}
        \caption*{(c) DDPM}
        \label{fig5:tieredImagenet}
    \end{minipage}
    \hfill
    \begin{minipage}{0.23\textwidth}
        \centering
        \includegraphics[width=\linewidth]{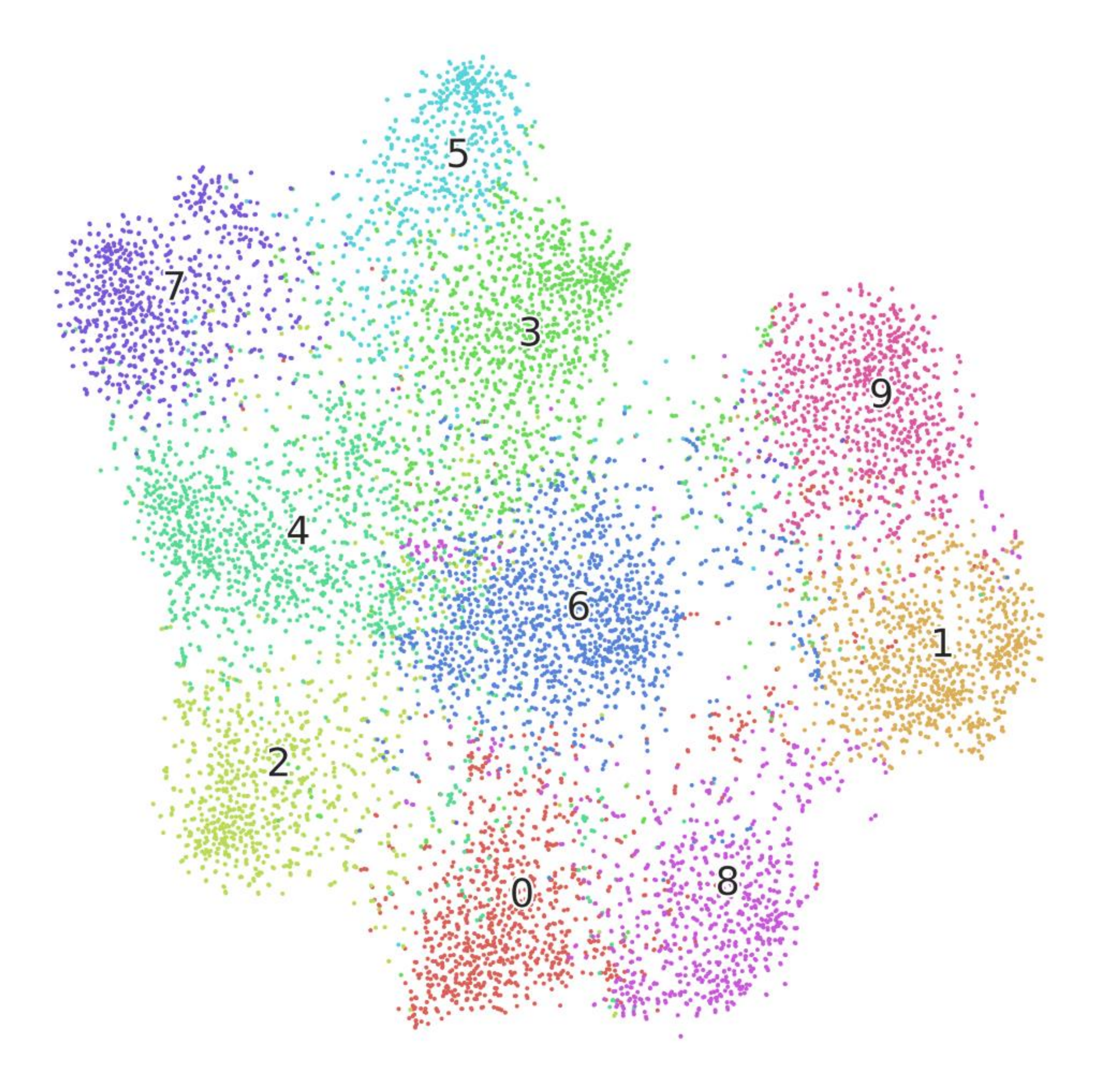}
        \caption*{(d) Ours}
        \label{fig5:subfig4}
    \end{minipage}
\caption{(a) Visualization of the representations learned by the network $\boldsymbol{F}$ for Cifar10 dataset. (b) Visualization of the $\mathbf{R}$ matrix. (c) and (d) Representations of images generated by DDPM model and our method.}
\label{fig5}
\end{figure}

For the quantitative experiment on the CelebA dataset, we also conduct experiments on the attributes: hair color, hair style, bangs type and young. The results are shown in Fig.\ref{fig_a1}. The DDPM still exhibits noticeable preferences, while our method approaches the distribution of the training dataset more closely. While our method doesn't achieve completely unbiased generation, it demonstrates effectiveness in addressing bias compared to DDPM.

The comparison of FID and sFID scores is shown in Table \ref{table1-1} and Table \ref{table2}. In general, the sFID scores of our method show a significant and consistent advantage over all the compared methods.

We present visualizations of the representations learned by the network $\boldsymbol{F}$ for Cifar10 in Fig.\ref{fig5}(a). We also visualize the $\mathbf{R}$ matrix for 100 randomly selected samples from the Cifar10 dataset in Fig.\ref{fig5}(b). Moreover, we provide visualizations of the generated image representations. The images are generated by the original DDPM model (Fig.\ref{fig5}(c)) and our method (Fig.\ref{fig5}(d)), respectively. These representations are extracted using a classifier trained on the Cifar10 dataset. Notably, the network $\mathbf{F}$ is capable of acquiring refined representations, with the $\mathbf{R}$ matrix closely resembling the ground truth. Furthermore, our method outperforms the original DDPM model, yielding more uniformly generated images.

\begin{small}
\begin{table*}[h]
\centering
\caption{Comparison of FID $\downarrow$ on four diverse datasets (the \textcolor{red}{red} numbers present our improvement)}
\renewcommand\arraystretch{1.5}
\begin{tabular}{cccccccccc}
\hline
Datasets                                       & Scheduler & 20                                      &50                                         & 100                                       & 1000                                       \\ \hline
\multirow{3}{*}{\textbf{Cifar10}}                       & \textbf{DDPM}                       & 133.37\textcolor{red}{-5.32}                                    & 32.72\textcolor{red}{-4.22}                                     & 9.99\textcolor{red}{-3.12}                        & 3.17\textcolor{red}{-0.42}             \\
                                               & \textbf{DDIM}                       & 10.9\textcolor{red}{-2.85}                     &  6.99\textcolor{red}{-2.37}                     &  5.52\textcolor{red}{-2.24}                     &  4.00\textcolor{red}{-1.23}                                           \\
                                               & \textbf{PNDM}                       & 5.00\textcolor{red}{-1.23}                     &  3.95\textcolor{red}{-0.85} & 3.72\textcolor{red}{-0.92} &  3.70\textcolor{red}{-0.78} \\ \hline
\multirow{3}{*}{\textbf{LSUN-church}}                   & \textbf{DDPM }                      & 12.47\textcolor{red}{-3.14}                                        & 10.84\textcolor{red}{-2.05}                                         & 10.58\textcolor{red}{-2.37}                                       & 7.89\textcolor{red}{-1.10}                                  \\
                                               & \textbf{DDIM}                       & 11.7\textcolor{red}{-2.89}                     &  10.0\textcolor{red}{-2.58}                     &  9.84\textcolor{red}{-2.05}                     &  9.88\textcolor{red}{-1.12}                     &                      \\
                                               & \textbf{PNDM}                       & 9.13\textcolor{red}{-1.24}  & 9.89\textcolor{red}{-1.55} & 10.1\textcolor{red}{-1.19} & 10.1\textcolor{red}{-1.22}  \\ \hline
\multirow{3}{*}{\textbf{LSUN-bedroom}}                  & \textbf{DDPM }                      & 22.77\textcolor{red}{-4.13}                                        & 10.81\textcolor{red}{-2.13}                                      & 6.81\textcolor{red}{-1.66}                           & 6.36\textcolor{red}{-1.23}                                         \\
                                               & \textbf{DDIM}                       & 8.47\textcolor{red}{-1.97}                     &  6.05\textcolor{red}{-2.11}                     &  5.97\textcolor{red}{-1.54}                     &  6.32\textcolor{red}{-1.10}                                         \\
                                               & \textbf{PNDM }                     & 5.68\textcolor{red}{-1.24}  & 6.44\textcolor{red}{-1.08} & 6.91\textcolor{red}{-1.12}  & 6.92\textcolor{red}{-0.91}  \\ \hline
\multicolumn{1}{c}{\multirow{2}{*}{\textbf{CelebA-HQ}}} & \textbf{DDIM}                       & 13.4\textcolor{red}{-2.75}  & 8.95\textcolor{red}{-2.20}  & 6.36\textcolor{red}{-1.47}  & 3.41\textcolor{red}{-0.79}  \\
\multicolumn{1}{l}{}                           & \textbf{PNDM }                      & 5.51\textcolor{red}{-1.11}                                         & 3.34\textcolor{red}{-1.03}                                        & 2.81\textcolor{red}{-0.56}                                        & 2.86\textcolor{red}{-0.73}                         \\ \hline
\end{tabular}
\label{table1-1}
\end{table*}
\end{small}

\begin{small}
\begin{table*}[]
\centering
\caption{Comparison of sFID $\downarrow$ on the four diverse Datasets (the \textcolor{red}{red} numbers present our improvement)}
\renewcommand\arraystretch{1.5}
\begin{tabular}{cccccccccc}
\hline
Datasets                                       & Scheduler & 20                                         & 50                                         & 100                                        & 1000                                       \\ \hline
\multirow{3}{*}{\textbf{Cifar10}}                       & \textbf{DDPM}                                    & 80.34\textcolor{red}{-10.79}                     & 40.18\textcolor{red}{-8.44}                                       & 12.23\textcolor{red}{-5.03}                               & 5.74\textcolor{red}{-0.89}                      \\
                                               & \textbf{DDIM}                           & 21.00\textcolor{red}{-5.78}                       & 11.31\textcolor{red}{-5.28}                & 9.88\textcolor{red}{-4.96}                    & 6.18\textcolor{red}{-2.04}                      \\
                                               & \textbf{PNDM}                                         & 8.31\textcolor{red}{-2.45}  & 6.70\textcolor{red}{-1.76}   & 6.53\textcolor{red}{-1.47}  & 6.50\textcolor{red}{-1.23}  \\ \hline
\multirow{3}{*}{\textbf{LSUN-church}}                   & \textbf{DDPM }                     & 24.19\textcolor{red}{-7.26}                             & 23.9\textcolor{red}{-4.23}                                        & 21.52\textcolor{red}{-5.56}                                       & 11.25\textcolor{red}{-2.07}                     \\
                                               & \textbf{DDIM}                                     & 22.69\textcolor{red}{-4.97}              & 22.13\textcolor{red}{-5.39}              & 18.23\textcolor{red}{-3.8}               & 18.35\textcolor{red}{-1.93}                     \\
                                               & \textbf{PNDM}                        & 18.03\textcolor{red}{-2.31} & 18.25\textcolor{red}{-2.73}  & 22.67\textcolor{red}{-1.93}  & 22.66\textcolor{red}{-1.94} \\ \hline
\multirow{3}{*}{\textbf{LSUN-bedroom}}                  & \textbf{DDPM }                             & 35.72\textcolor{red}{-8.29}                 & 23.84\textcolor{red}{-4.34}                          & 11.31\textcolor{red}{-4.28}      & 11.07\textcolor{red}{-2.04}                     \\
                                               & \textbf{DDIM}                                         & 15.86\textcolor{red}{-3.27}              & 10.73\textcolor{red}{-3.04}              & 10.35\textcolor{red}{-2.67}              & 10.97\textcolor{red}{-1.78}                     \\
                                               & \textbf{PNDM }                      & 10.12\textcolor{red}{-2.43} & 11.20\textcolor{red}{-2.29} & 11.79\textcolor{red}{-1.86} & 11.80\textcolor{red}{-1.16} \\ \hline
\multicolumn{1}{c}{\multirow{2}{*}{\textbf{CelebA-HQ}}} & \textbf{DDIM}                       & 26.09\textcolor{red}{-4.69} & 13.22\textcolor{red}{-4.14}  & 11.23\textcolor{red}{-3.23}  & 5.63\textcolor{red}{-1.42} \\
\multicolumn{1}{l}{}                           & \textbf{PNDM }                                 & 9.82\textcolor{red}{-1.78}                                        & 5.53\textcolor{red}{-1.49}                                         & 4.69\textcolor{red}{-1.04}                                     & 4.73\textcolor{red}{-1.55}                      \\ \hline
\end{tabular}
\label{table2}
\end{table*}
\end{small}

\begin{table}
\centering
\caption{Ablation study on the relationship matrices $C$}
\renewcommand\arraystretch{1.2}
\resizebox{\linewidth}{!}{
\begin{tabular}{ccccc}
\hline
\textbf{Method}            & \textbf{Cifar10} & \textbf{Church}      & \textbf{Bedroom}     & \textbf{CelebA} \\ \hline
\textbf{Baseline}          & 6.99             & 10.0                      & 6.05                      & 8.95               \\
\textbf{K-means(10)}       & 6.28             & \multicolumn{1}{c}{8.26} & 4.73                      & 7.26               \\
\textbf{K-means(20)}       & 6.49             & \multicolumn{1}{c}{8.79} & 5.24                      & 7.22             \\ \hline
\makecell{\textbf{learnable} \\ \textbf{(SimCLR)}} & \textbf{4.62}             & \multicolumn{1}{c}{7.42} & 3.94                      & \textbf{6.75}               \\
\makecell{\textbf{learnable} \\ \textbf{(CLIP)}}   & 4.74             & 7.73                      & \multicolumn{1}{c}{4.16} & 6.77               \\
\makecell{\textbf{learnable} \\ \textbf{(ResNet)}} & 4.92             & \textbf{7.33}                     & \textbf{3.82}                      & 6.89               \\ \hline
\end{tabular}}
\label{table3}
\end{table}

\subsubsection{Ablation Study}
In this section, we provide ablation and analysis over different components of our method. 

\textbf{Relationship Matrices $\mathbf{C}$}. We offer several implementation methods for obtaining the relationship matrices $\mathbf{C}$ in $\mathcal{L}_{M}$. The first approach utilizes the contrastive learning model SimCLR \cite{simclr} in combination with K-means clustering. We employ a pre-trained SimCLR model to extract features from the training dataset, and then perform K-means clustering on these features. The resulting clustering centroids serve as prototypes of the submanifolds. In this experiment, we consider using 10 and 20 clustering centroids as options. The second approach introduces learnable matrices, which is proposed in this paper. We choose the pre-trained encoders $\boldsymbol{E}_{pre}$ to be SimCLR, CLIP \cite{clip}, and ResNet50 \cite{he2016deep}  (pre-trained on the ImageNet). We use the DDPM model with DDIM scheduler for experiments and set the sampling time steps to 50. Table \ref{table3} presents the FID scores of the ablation study conducted on the Cifar10, CelebA-HQ, LSUN-church, and LSUN-bedroom datasets. It can be observed that the choice of the pre-trained encoder has minimal impact on the results. For our subsequent experiments, we opt for the learnable approach using the SimCLR model as the pre-trained encoder.
%To conduct the ablation study, 

\textbf{Hyperparameters}. We also conduct ablation studies on the number of guidance steps. The experiment results on the LSUN-church, LSUN-bedroom, CelebA-HQ datasets with 50 sampling time steps are shown in Table \ref{table5-1}, Table \ref{table6} and Table \ref{table7}. We apply $MGS$ only on the first few sampling steps. It can be seen that different schedulers have different optimal guidance steps. We choose the best guidance steps for each scheduler.

\begin{table}
\centering
\caption{Ablation study on the LSUN-church dataset}
\renewcommand\arraystretch{1.2}
\begin{tabular}{ccccc}
\hline
\textbf{Guidance steps} & 5             & 10                                 & 20   & 50    \\ \hline
\textbf{DDIM}      & \textbf{7.42} & 7.57                               & 8.46 & 9.24  \\
\textbf{PNDM}      & 8.61          & \multicolumn{1}{c}{\textbf{8.34}} & 9.23 & 9.82  \\
\textbf{DDPM}      & 9.25          & \multicolumn{1}{c}{\textbf{8.79}} & 9.73 & 10.14 \\ \hline
\end{tabular}
\label{table5-1}
\end{table}

\begin{table}
\centering
\caption{Ablation study on the LSUN-bedroom dataset}
\renewcommand\arraystretch{1.2}
\begin{tabular}{ccccc}
\hline
\textbf{Scheduler} & 5             & 10            & 20   & 50   \\ \hline
\textbf{DDIM}      & \textbf{3.94} & 4.13          & 4.87 & 5.18 \\
\textbf{PNDM}      & 5.67          & \textbf{5.36} & 6.22 & 6.36 \\
\textbf{DDPM}      & 8.03          & \textbf{7.68} & 8.45 & 9.79 \\ \hline
\end{tabular}
\label{table6}
\end{table}

\begin{table}
\centering
\caption{Ablation study on the CelebA-HQ dataset}
\renewcommand\arraystretch{1.2}
\begin{tabular}{ccccc}
\hline
\textbf{Scheduler} & 5             & 10            & 20   & 50   \\ \hline
\textbf{DDIM}      & \textbf{6.75} & 6.83          & 7.25 & 7.92 \\
\textbf{PNDM}      & 2.76          & \textbf{2.31} & 3.22 & 3.43 \\ \hline
\end{tabular}
\label{table7}
\end{table}

\vspace{-0.2cm}
\subsubsection{Visual Results}
The visual results of different datasets are shown in Fig.\ref{fig6}, Fig.\ref{cat}, Fig.\ref{church} and Fig.\ref{bedroom}, respectively. Our method outperforms DDPM on image diversity and quality.

\begin{figure*}[h]
\centering
\includegraphics[width=0.8\textwidth]{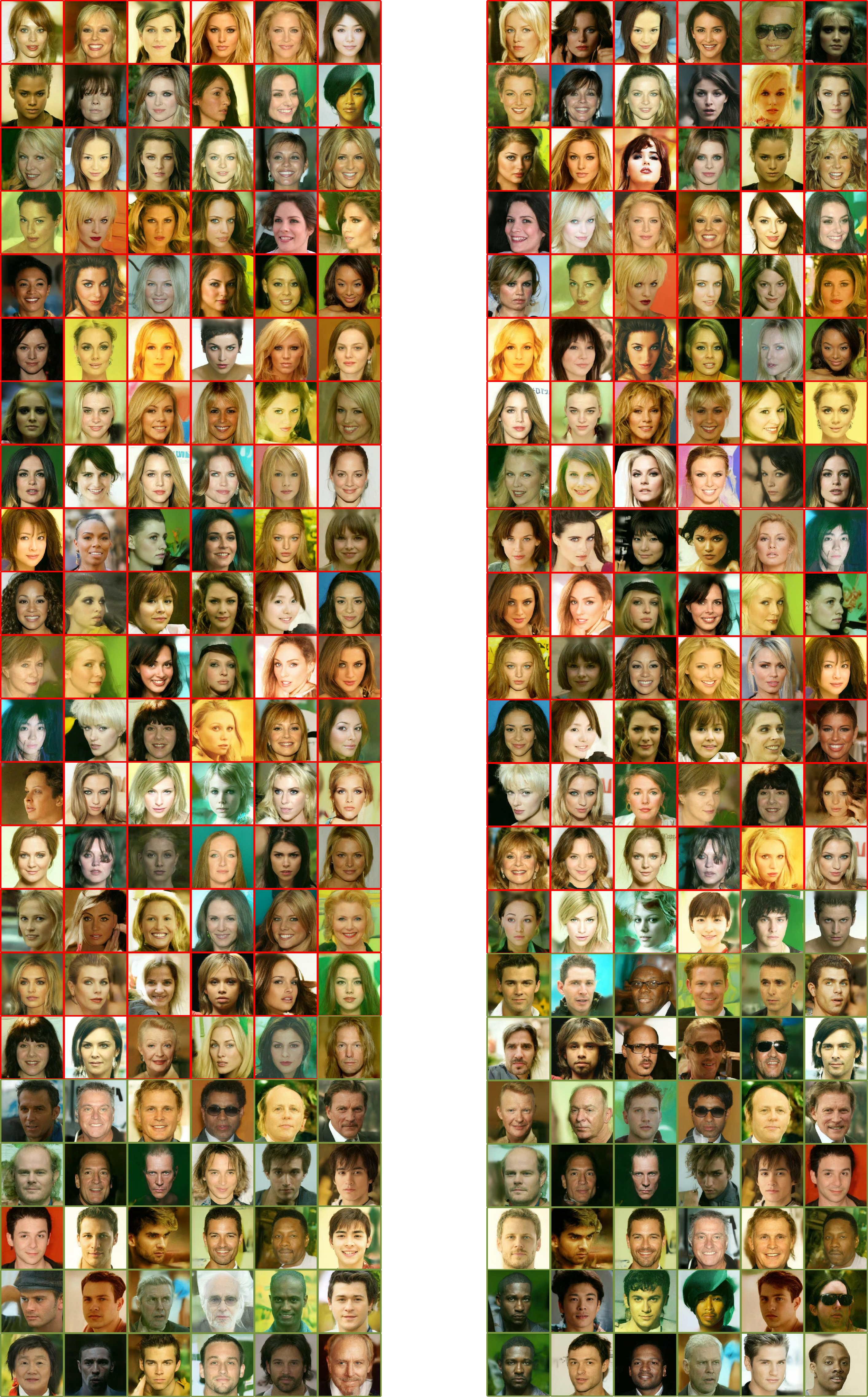}
\caption{Random batch of samples generated from DDPM (\textbf{left}) and $MGS$ (\textbf{right}) on the CelebA-HQ dataset. Samples are sorted and color coded by gender.}
\label{fig6}
\end{figure*}

\begin{figure*}[h]
\centering
\includegraphics[width=1\textwidth]{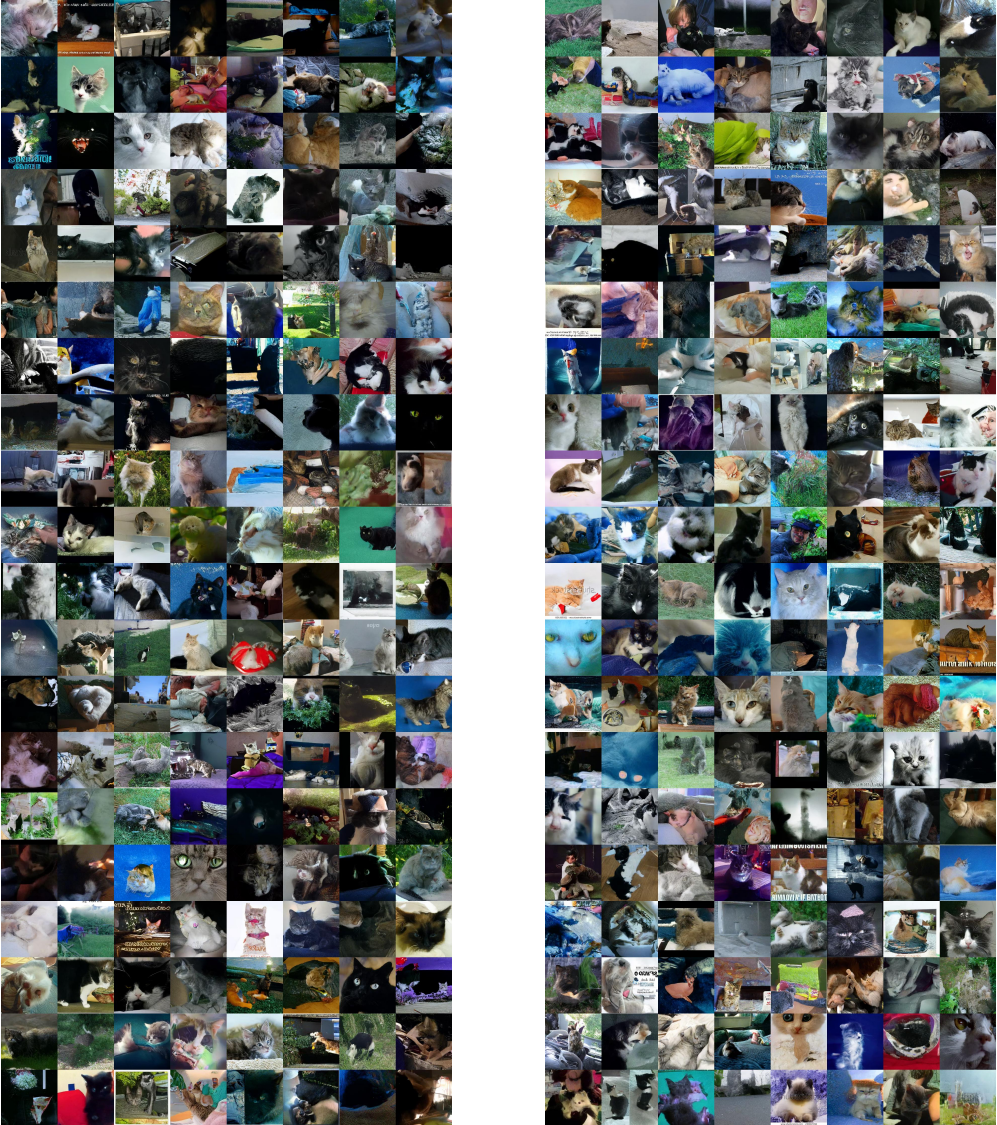}
\caption{Random batch of samples generated from DDPM (\textbf{left}) and $MGS$ (\textbf{right}) on the LSUN-cat dataset.}
\label{cat}
\end{figure*}

\begin{figure*}[h]
\centering
\includegraphics[width=1\textwidth]{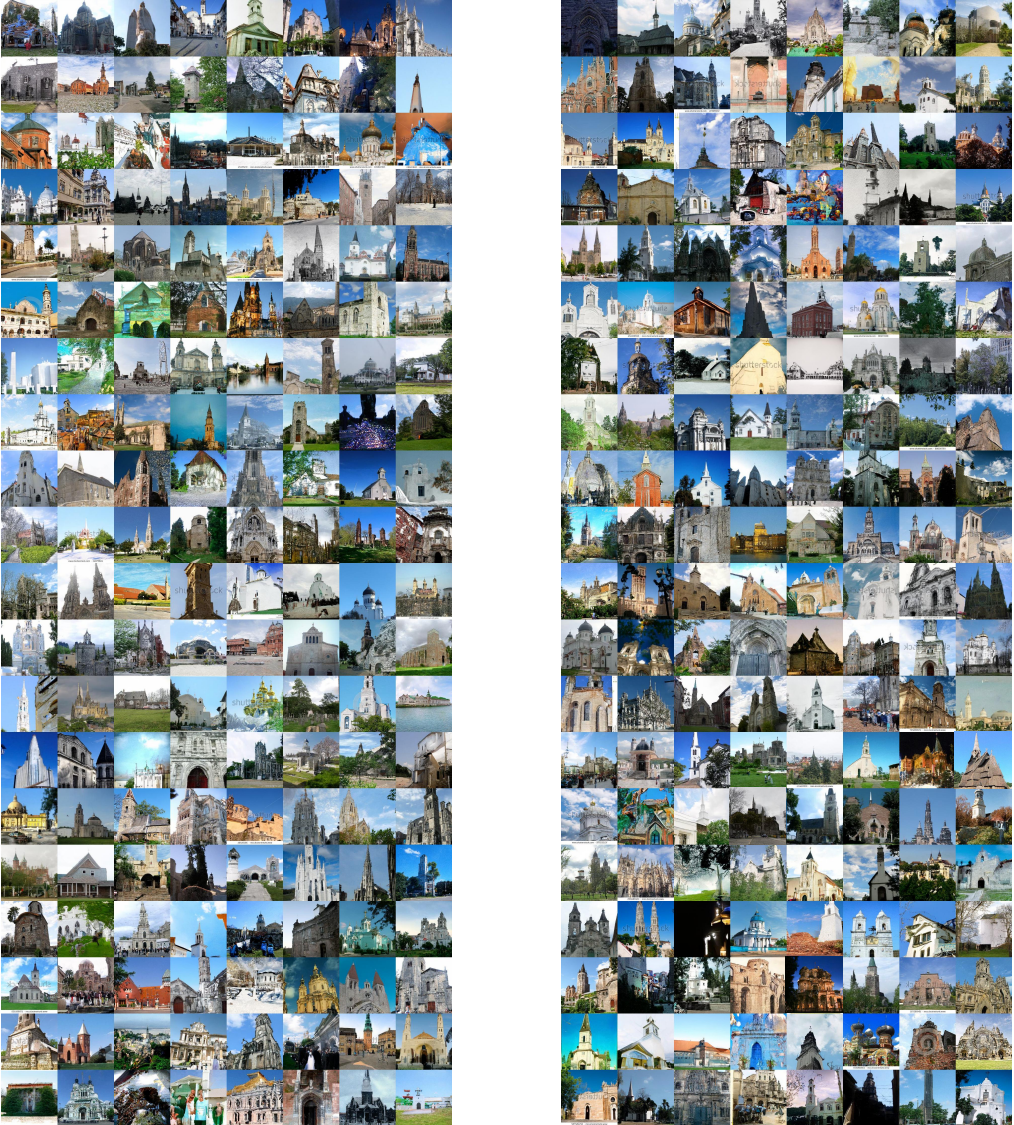}
\caption{Random batch of samples generated from DDPM (\textbf{left}) and $MGS$ (\textbf{right}) on the LSUN-church dataset.}
\label{church}
\end{figure*}

\begin{figure*}[h]
\centering
\includegraphics[width=1\textwidth]{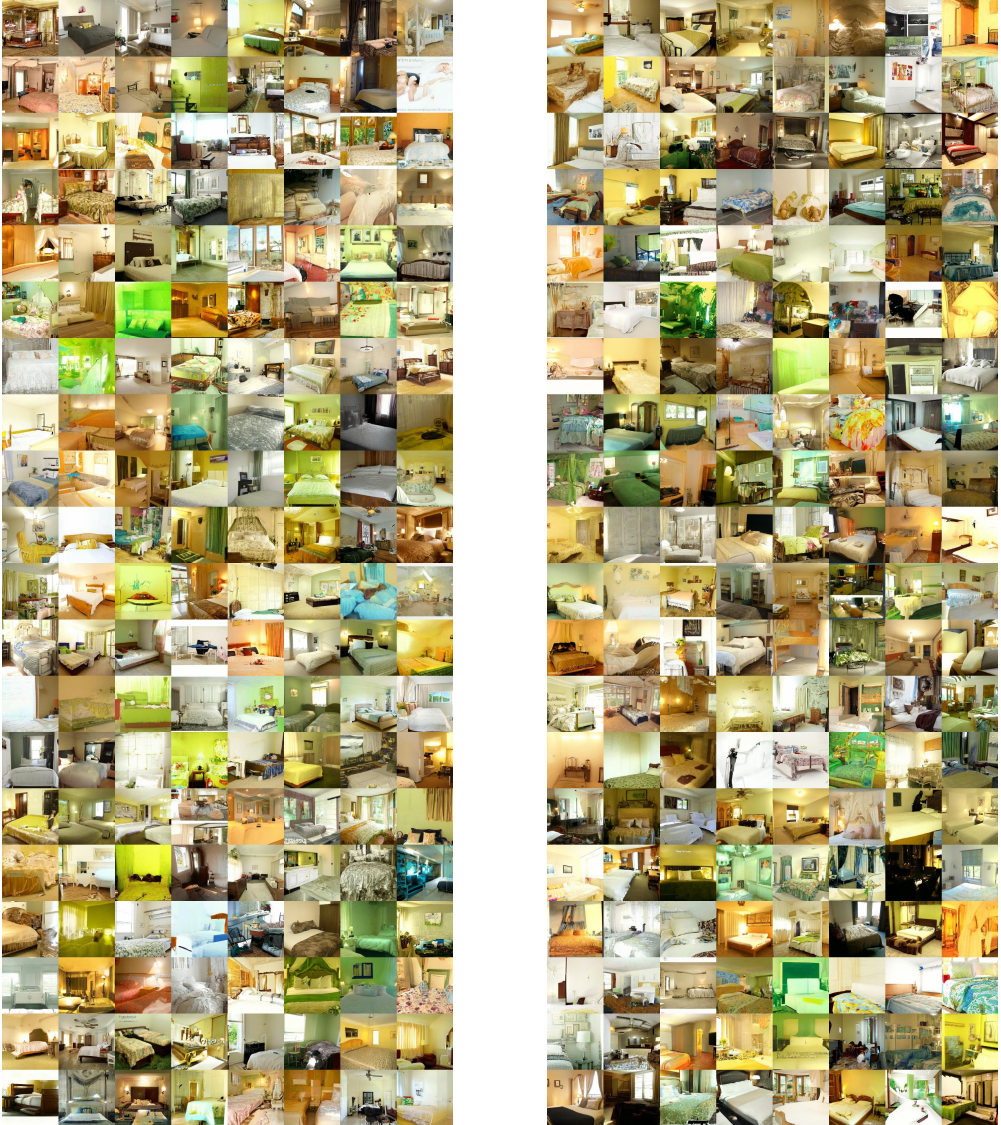}
\caption{Random batch of samples generated from DDPM (\textbf{left}) and $MGS$ (\textbf{right}) on the LSUN-bedroom dataset.}
\label{bedroom}
\end{figure*}

\vspace{-0.2cm}
\subsection{Related Works}
\label{rw}

Diffusion models are a family of probabilistic generative models that can produce high-quality samples from complex data domains. There are three main formulations of diffusion models: denoising diffusion probabilistic models (DDPMs) \cite{ho2020denoising,songdenoising}, score-based generative models (SGMs) \cite{song2019generative,song2020improved}, and stochastic differential equations (Score SDEs) \cite{song2020score,song2021maximum}. A major challenge of diffusion models is their slow and costly sampling process \cite{yang2022diffusion}, which has motivated several research directions to improve their efficiency and accuracy \cite{songdenoising,luhman2021knowledge,liupseudo,rombach2022high,ludpm,lyu2022accelerating,kingma2021variational,nichol2021improved,lu2022maximum}. Another research direction is to adapt diffusion models to data with special structures or properties, such as permutation invariance \cite{xugeodiff,shi2021learning}, manifold structures \cite{vahdat2021score,deriemannian,huang2022riemannian}, and discrete data \cite{austin2021structured,mengconcrete,liu2023learning}.

In this paper, we address a critical yet overlooked issue in current research: the bias issue in generated samples. A recurring phenomenon in the outputs of DDPMs is the propensity for these samples to aggregate around particular modes of the training data, thereby engendering a non-uniform distribution across the data manifold. Furthermore, most diffusion models are trained on web-scraped image datasets, which may contain implicit or explicit biases that reflect the skewed or imbalanced patterns of the data collection process. Such biases can adversely affect the quality and diversity of the generated samples, as well as downstream tasks that rely on them. Previous works \cite{orgad2023editing, friedrich2023fair,seshadri2023bias} also identify a similar issue with text-to-image diffusion models, and propose a method to edit implicit assumptions in the pre-trained models. However, their methods require prior knowledge of attributes the data exhibit bias in. Acquiring such priors can be challenging, whereas our method does not rely on this prior. Moreover, the authors in \cite{chungimproving} propose to add manifold constraint guidance on conditional diffusion models. Their method relies on predefined prompts and focuses more on inverse problems, whereas our method is unsupervised and aims to alleviate data bias problem of image generation. PNDM \cite{liupseudo} proposes a pseudo numerical method to generate sample along a specific manifold in $\mathcal{R}^N$. This method aims to reduce the noise introduced by numerical methods and the manifold is defined by a single sample, whereas our method estimates the manifold based on training data and aims to reduce the data bias in sampling.

\end{document}